\documentclass[number]{ReportTemplate}
\usepackage{utopia}
\usepackage{amssymb}
\usepackage{amsfonts}
\usepackage{amsmath}
\usepackage{mathtools}
\usepackage{amsthm}
\usepackage{bm}
\usepackage{enumitem}
\usepackage{graphicx}
\usepackage{algorithm}
\usepackage{algorithmic}
\usepackage[american]{babel}
\usepackage{subfigure}
\usepackage{hyperref}
\mathtoolsset{showonlyrefs,showmanualtags}

\usepackage{multirow}

\newtheorem{remark}{Remark}

\begin{document}
\begin{frontmatter}

\title{Multi-objective Evolutionary Algorithms are Generally Good: Maximizing Monotone Submodular Functions over Sequences}
\author{Chao Qian$^{1}$}
\ead{qianc@lamda.nju.edu.cn}
\author{Dan-Xuan Liu$^{1}$}
\ead{liudx@lamda.nju.edu.cn}
\author{Chao Feng$^{2}$}
\ead{chaofeng@mail.ustc.edu.cn}
\author{Ke Tang$^{3}$\corref{cor1}}
\ead{tangk3@sustech.edu.cn}
\cortext[cor1]{Corresponding author. A preliminary version of this paper has appeared at IJCAI'18~\cite{qian2018sequence}.}
\address{$^{1}$State Key Laboratory for Novel Software Technology, Nanjing University, Nanjing 210023, China\\\vspace{0.5em}
$^{2}$School of Computer Science and Technology, University of Science and Technology of China, Hefei 230027, China\\\vspace{0.5em}
$^{3}$Shenzhen Key Laboratory of Computational Intelligence, Department of Computer Science and Engineering,\\ Southern University of Science and Technology, Shenzhen 518055, China}

\begin{abstract}
Evolutionary algorithms (EAs) are general-purpose optimization algorithms, inspired by natural evolution. Recent theoretical studies have shown that EAs can achieve good approximation guarantees for solving the problem classes of submodular optimization, which have a wide range of applications, such as maximum coverage, sparse regression, influence maximization, document summarization and sensor placement, just to name a few. Though they have provided some theoretical explanation for the general-purpose nature of EAs, the considered submodular objective functions are defined only over sets or multisets. To complement this line of research, this paper studies the problem class of maximizing monotone submodular functions over sequences, where the objective function depends on the order of items. We prove that for each kind of previously studied monotone submodular objective functions over sequences, i.e., prefix monotone submodular functions, weakly monotone and strongly submodular functions, and DAG monotone submodular functions, a simple multi-objective EA, i.e., GSEMO, can always reach or improve the best known approximation guarantee after running polynomial time in expectation. Note that these best-known approximation guarantees can be obtained only by different greedy-style algorithms before. Empirical studies on various applications, e.g., accomplishing tasks, maximizing information gain, search-and-tracking and recommender systems, show the excellent performance of the GSEMO.
\end{abstract}

\end{frontmatter}

\newpage
\section{Introduction}

Evolutionary algorithms (EAs)~\citep{back:96} are a large class of heuristic randomized optimization algorithms, inspired by natural evolution. They simulate the natural evolution process by considering two key factors, variational reproduction and superior selection. Generally, EAs maintain a population of solutions, and repeatedly reproduce solutions and eliminate inferior ones so that the maintained solutions can be improved iteratively. In order to solve an optimization problem, EAs only require the solutions to be evaluated, and thus are general-purpose optimization algorithms. As a result, EAs have been successfully applied to solve various sophisticated optimization problems, e.g., antenna design~\citep{hornby2011computer}, neural architecture search~\citep{elsken2019neural} and biodiversity analysis~\citep{fan2020high}, just to name a few.

However, in contrast to the great success in practice, the theoretical foundations of EAs are still underdeveloped. Particularly, most of previous theoretical analyses considered isolated problems~\citep{auger2011theory,neumann2020discrete,neumann2010bioinspired}, which cannot reflect the general-purpose property of EAs. With the goal of showing good general approximation ability of EAs theoretically, some efforts~\citep{bian2020efficient,do2020maximizing,aaai2021dynamic,friedrich2015maximizing,friedrich2018heavy,neumann2020optimising,qian2020multi,qian2017generalsubset,qian2017constrained,qian2018multiset,qian2019maximizing,aaai2019dynamic} have been recently put into studying their approximation performance for solving general problem classes of submodular optimization, where the objective function is only required to satisfy the submodular property and its explicit formulation is not needed. Submodularity~\citep{nemhauser1978analysis} characterizes the diminishing returns property, and is satisfied by the objective functions of many applications, e.g., maximum coverage~\citep{feige1998threshold,liu2023primal}, feature selection~\citep{6137220}, influence maximization~\citep{kempe2003maximizing} and document summarization~\citep{lin2011class}, to name a few. Thus, submodular optimization is a fundamental task and has attracted a lot of research attention~\citep{krause2014submodular}. As it is NP-hard in general, many algorithms with bounded approximation guarantees have been developed.

Submodular optimization was originally based on set functions. That is, the objective function $f$ is a set function, which maps any subset of a given ground set $V$ to a real value. Let $\mathbb{R}$ denote the set of reals. A set function $f:2^V \rightarrow \mathbb{R}$ is submodular if $\forall X \subseteq Y \subseteq V, v \notin Y: f(X \cup \{v\})-f(X) \geq f(Y \cup \{v\}) - f(Y)$, i.e., the benefit of adding an item to a set will not increase as the set extends, implying the diminishing returns property. For maximizing monotone submodular functions with a size constraint, i.e.,
\begin{align}\label{eq-subset}
\arg \max\nolimits_{X \subseteq V}f(X) \quad \text{s.t.} \quad |X| \leq k,
\end{align}
where $f$ is monotone and submodular, Friedrich and Neumann~\cite{friedrich2015maximizing} proved that the GSEMO, a simple multi-objective EA widely used in theoretical analyses~\citep{Laumanns04}, can achieve the optimal polynomial-time approximation ratio of $1-1/e$~\citep{nemhauser1978best,nemhauser1978analysis} in $O(n^2(\log n +k))$ expected running time, where $n$ is the size of $V$. Note that a set function $f: 2^V \rightarrow \mathbb{R}$ is monotone if $\forall X \subseteq Y \subseteq V: f(X) \leq f(Y)$. When $f$ is monotone and approximately submodular, i.e., satisfies the submodular property to some extent, Qian et al.~\cite{qian2019maximizing} proved that the GSEMO can achieve the optimal polynomial-time approximation ratio of $1-e^{-\gamma}$~\citep{das2011submodular,harshaw2019submodular}, where $\gamma \in [0,1]$ measures the closeness of $f$ to submodularity. When $f$ is submodular and approximately monotone, Qian et al.~\cite{qian2019maximizing} also proved that the GSEMO can find a subset $X$ with $f(X) \geq (1-1/e)\cdot (\mathrm{OPT}-k\epsilon)$ in $O(n^2(\log n +k))$ expected running time, where $\mathrm{OPT}$ denotes the optimal function value, and $\epsilon \geq 0$ captures the degree of approximate monotonicity. This reaches the best-known polynomial-time approximation guarantee~\citep{krause2008near}. The good approximation performance of EAs has also shown on the problem of Eq.~(\refeq{eq-subset}) with non-monotone objective functions~\citep{friedrich2018heavy,qian2020multi}, noisy objective functions~\citep{qian2019distributed,qian2017subset}, general cost constraints~\citep{bian2020efficient,qian2017generalsubset}, chance constraints~\citep{neumann2020optimising}, or partition matroid constraints~\citep{do2020maximizing}. Furthermore, it has been proved that EAs can regain the good approximation guarantee efficiently even when the objective functions or constraints change dynamically~\cite{bian2021dynamic,aaai2021dynamic,qian2022result,aaai2019dynamic}.

Note that a subset $X \subseteq V=\{v_1,v_2,\ldots,v_n\}$ can be naturally represented by a Boolean vector $\bm{x} \in \{0,1\}^n$, where the $i$-th bit $x_i=1$ if the $i$-th item $v_i \in X$, otherwise $x_i=0$. Thus, a set function $f: 2^{V} \rightarrow \mathbb{R}$ is actually a pseudo-Boolean function $f: \{0,1\}^n \rightarrow \mathbb{R}$. Submodularity has been extended from pseudo-Boolean functions to functions $f:\mathbb{Z}_{+}^V \rightarrow \mathbb{R}$ over the integer lattice $\mathbb{Z}_{+}^V$~\citep{soma2016maximizing,ward2014maximizing}, where $\mathbb{Z}_{+}$ denotes the set of non-negative integers. Consider the problem corresponding to Eq.~(\refeq{eq-subset}), which now changes to
\begin{align}\label{eq-multiset}
\arg \max\nolimits_{\bm{x} \in \mathbb{Z}_{+}^V}f(\bm{x}) \quad \text{s.t.} \quad |\bm{x}| \leq k.
\end{align}
An integer vector $\bm{x} \in \mathbb{Z}_{+}^V$ can represent a multiset, where the item $v_i$ appears $x_i$ times. In this case, when the objective function is monotone and DR-submodular, and $|\bm{x}|$ represents the $\ell_1$-norm of $\bm{x}$, Qian et al.~\cite{qian2018multiset} proved that the GSEMO can achieve the best-known polynomial-time approximation ratio of $1-1/e$~\citep{soma2014optimal}. The GSEMO employs the bit-wise mutation operator, which flips the value on each dimension of a vector independently with probability $1/n$, to reproduce offspring solutions. Note that for a Boolean vector, the behavior of flipping changes 0 to 1 or 1 to 0, while for an integer vector, it changes the current value to a different integer selected uniformly at random. An integer vector $\bm{x} \in \mathbb{Z}_{+}^V$ can also represent $k$ subsets, where the $i$-th subset is $\{v_j \mid x_j=i\}$. In this case, when the objective function is monotone and $k$-submodular, and $|\bm{x}|$ represents the $\ell_0$-norm of $\bm{x}$, Qian et al.~\cite{qian2017constrained} proved that the GSEMO combined with randomized local search can achieve the asymptotically optimal polynomial-time approximation ratio of $1/2$~\citep{ohsaka2015monotone}.

In many practical applications such as job scheduling~\citep{stadje1995selecting} and recommender systems~\citep{ashkan2015optimal}, it is often desired to select a sequence instead of a subset. That is, the items cannot be treated independently, and the objective function $f$ depends on the order of items. Thus, submodularity has also been extended to functions $f:\mathcal{S} \rightarrow \mathbb{R}$ over sequences, where $\mathcal{S}=\{(s_1,s_2,\ldots,s_l) \mid s_i \in V, l \in \mathbb{Z}_{+}\}$ denotes the space of sequences. A natural question is then whether EAs can achieve good polynomial-time approximation guarantees for submodular optimization over sequences, which has not been touched before, and is the main focus of this paper.

For submodular optimization over sequences, existing studies have only considered the problem corresponding to Eq.~(\refeq{eq-subset}), which is
\begin{align}\label{eq-sequence}
\mathop{\arg\max}\nolimits_{s \in \mathcal{S}} f(s) \quad \text{s.t.}\quad |s|\leq k,
\end{align}
i.e., to select a sequence of at most $k$ items that will maximize some given objective function $f$. Here $|s|$ represents the length of a sequence $s$. Note that the search space is exponentially larger by considering sequences instead of subsets, and thus Eq.~(\refeq{eq-sequence}) can be harder than Eq.~(\refeq{eq-subset}). Alaei et al.~\cite{alaei2021maximizing} first proved that for \emph{subsequence monotone submodular} objective functions, the greedy algorithm, which iteratively appends one item with the largest marginal gain on $f$ to the end of the current sequence, can achieve a $(1-1/e)$-approximation guarantee. When the objective function $f$ is relaxed to be \emph{prefix$+$suffix monotone} and \emph{prefix submodular}, Streeter and Golovin~\cite{streeter2008online} proved that the greedy algorithm can still achieve the $(1-1/e)$-approximation guarantee. When $f$ is further relaxed to be \emph{prefix monotone submodular}, Zhang et al.~\cite{zhang2016string} proved that the greedy algorithm achieves a $(1/\sigma)(1-e^{-\sigma})$-approximation guarantee, where $\sigma$ is the curvature characterizing the degree of submodularity. Recently, Bernardini et al.~\cite{bernardini2020through} proved that when $f$ is \emph{weakly monotone} and \emph{strongly submodular}, the greedy algorithm fails to achieve a constant approximation guarantee, and proposed the generalized greedy algorithm, which can achieve the approximation guarantee of $1-1/e$. Different from the greedy algorithm which iteratively appends an item to the end of the current sequence, the generalized greedy algorithm can insert an item into any position of the current sequence in each iteration. Note that weak monotonicity is weaker than prefix monotonicity while strong submodularity is stronger than subsequence submodularity. Tschiatschek et al.~\cite{tschiatschek2017selecting} considered another class of objective functions, so-called \emph{DAG monotone submodular} functions, which must satisfy the subsequence monotone property but does not necessarily satisfy the prefix submodular property. They proved that the greedy algorithm also fails to achieve a constant approximation guarantee, and proposed a new algorithm \textsc{OMegA} with an approximation guarantee of $1-e^{-1/(2\Delta)}$, where $\Delta \geq 1$. The relationship among the above five function classes over sequences is shown in Figure~\ref{fig_relation-mon-sub}, which will be explained in detail in Section~\ref{sec-problemclass}.

\begin{figure*}[ht!]\centering
\begin{minipage}[c]{0.45\linewidth}\centering
        \includegraphics[width=1\linewidth]{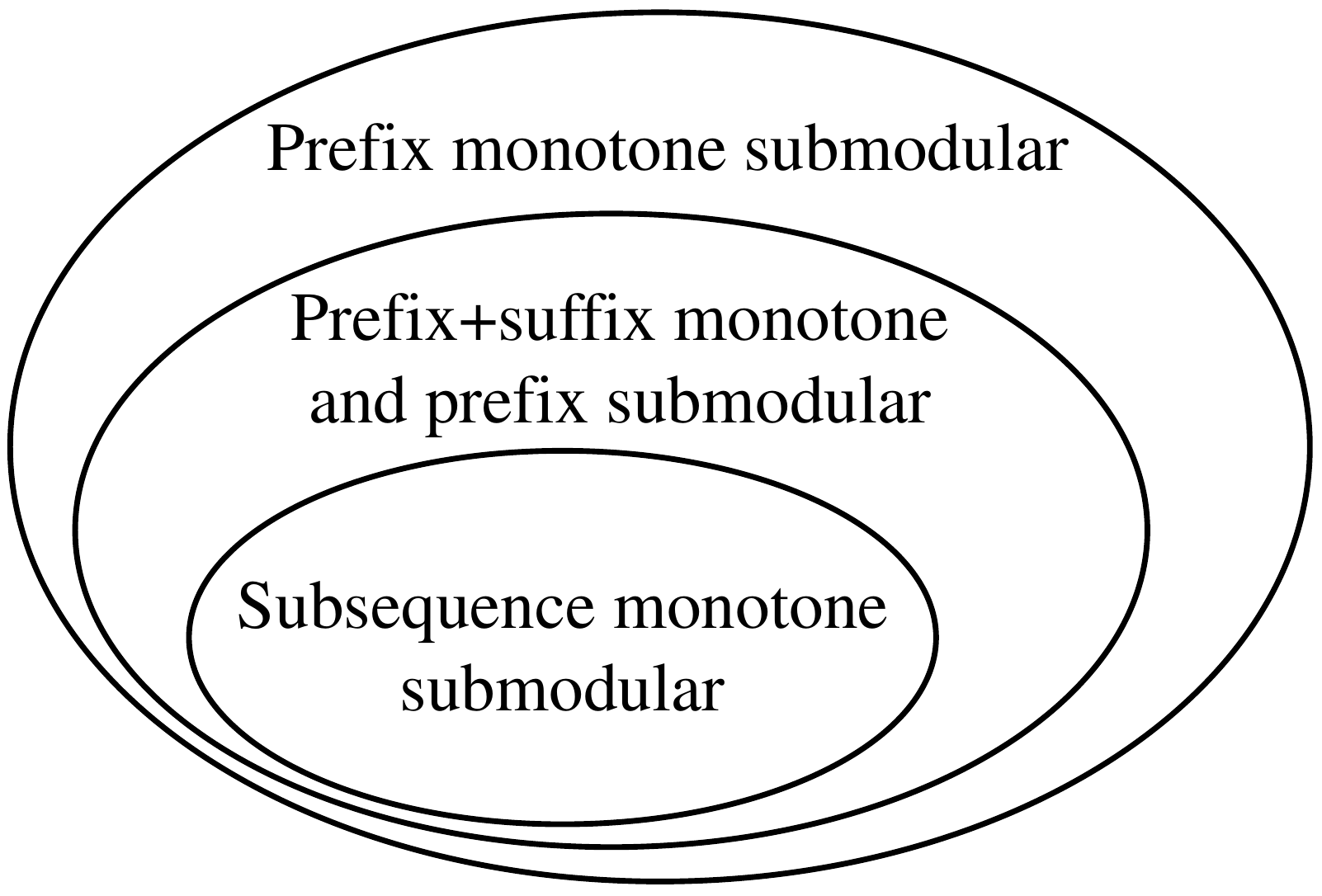}
\end{minipage}\hspace{2.5em}
\begin{minipage}[c]{0.35\linewidth}\centering
        \includegraphics[width=1\linewidth]{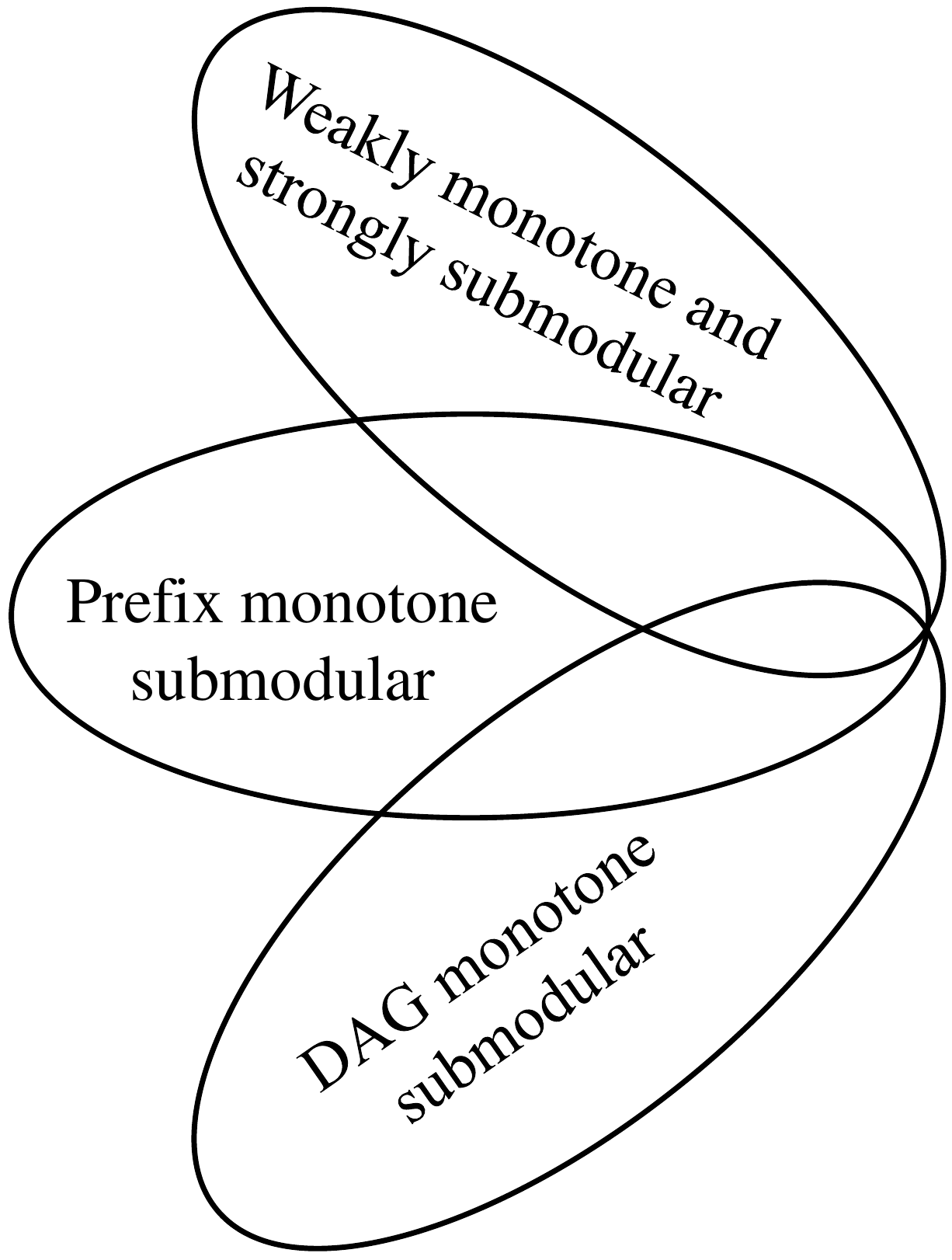}
\end{minipage}
\caption{The relationship among the monotone submodular function classes over sequences~\citep{alaei2021maximizing,bernardini2020through,streeter2008online,tschiatschek2017selecting,zhang2016string} studied before.}\label{fig_relation-mon-sub}
\end{figure*}

In this paper, to theoretically examine the performance of EAs for submodular optimization over sequences, we analyze the approximation guarantee of the GSEMO for solving the problem of Eq.~(\refeq{eq-sequence}), by considering each class of previously studied monotone submodular objective functions. The GSEMO solves the reformulated bi-objective optimization problem that maximizes the given objective $f(s)$ and minimizes the sequence length $|s|$ simultaneously, and outputs the best sequence satisfying the length constraint from the final population. To make the GSEMO able to optimize functions over sequences, the mutation operator is changed accordingly. Inspired by~\citep{durrett2011computational,scharnow2005analysis}, it selects a number $r$ randomly from the Poisson distribution with $\lambda=1$, and then inserts or deletes one item randomly for $r$ times. This operator can be viewed as a natural extension of bit-wise mutation over Boolean vectors. We prove that the GSEMO using polynomial expected running time can always reach or improve the best known approximation guarantee. The concrete theoretical results are:\vspace{-1em}
\begin{itemize}[leftmargin=1.8em]
\item[(1)] When $f$ is prefix monotone submodular, the GSEMO using at most $2ek^2(k+1)n$ expected running time can achieve an approximation guarantee of $(1/\sigma)(1-e^{-\sigma})$ (\textbf{Theorem~\ref{theo-prefix}}), which reaches that of the greedy algorithm~\citep{zhang2016string}.
    \begin{itemize}[leftmargin=*]
    \item When $f$ is specialized to be prefix$+$suffix monotone and prefix submodular, the GSEMO using at most $2ek^2(k+1)n$ expected running time can achieve an approximation guarantee of $1-1/e$ (\textbf{Theorem~\ref{theo-prefix-suffix}}), which reaches that of the greedy algorithm~\citep{streeter2008online}.
    \item When $f$ is further specialized to be subsequence monotone submodular, the GSEMO using at most $2ek^2(k+1)n$ expected running time can achieve an approximation guarantee of $1-1/e$ (\textbf{Theorem~\ref{theo-subsequence}}), which reaches that of the greedy algorithm~\citep{alaei2021maximizing}.
    \end{itemize}
\item[(2)] When $f$ is weakly monotone and strongly submodular, the GSEMO using at most $2ek^2(k+1)n$ expected running time can achieve an approximation guarantee of $1-1/e$ (\textbf{Theorem~\ref{theo-weak}}), which reaches that of the generalized greedy algorithm~\citep{bernardini2020through}.
\item[(3)] When $f$ is DAG monotone submodular, the GSEMO using at most $4ek^2n^2$ expected running time can achieve an approximation guarantee of $1-e^{-1/2}$ (\textbf{Theorem~\ref{theo-dag}}), which is at least as good as that of the \textsc{OMegA} algorithm, i.e., $1-e^{-1/(2\Delta)}$~\citep{tschiatschek2017selecting}, where $\Delta \geq 1$.
\end{itemize}
Note that for the last two cases, the greedy algorithm fails to achieve a constant approximation guarantee~\citep{bernardini2020through,tschiatschek2017selecting}. Thus, these ``one-for-all" theoretical results of the GSEMO have shown the good general approximation ability of EAs for monotone submodular optimization over sequences. The analysis also discloses the importance of the mutation operator, which can simulate various greedy operators and thus leads to the universality of EAs.

We also examine the performance of the GSEMO by experiments. For each of the above three studied problem classes, we consider two applications, and compare the GSEMO with the greedy algorithm, the generalized greedy algorithm, and the \textsc{OMegA} algorithm. The experimental results show that the objective value achieved by the GSEMO is always at least as large as that achieved by the best previous algorithm, and the GSEMO is actually significantly better in most cases by the \textit{sign-test}~\citep{demsar:06} with confidence level $0.05$. Furthermore, the GSEMO can bring a performance improvement even when the previous algorithm has been nearly optimal. The results also show that, compared with the running time bound (i.e., the worst-case running time) derived in the theoretical analyses, the GSEMO can be much more efficient in practice.

This paper extends our preliminary work~\citep{qian2018sequence}. In the theoretical sections, we add the analysis of the GSEMO for solving the problem class of maximizing weakly monotone and strongly submodular functions (i.e., Section~\ref{sec-theory2}). In the experimental section, for the problem class of maximizing prefix monotone submodular functions, we only considered the application of accomplishing tasks (i.e., Section~\ref{subsec-exp-1-1}), and now add the application of maximizing information gain (i.e., Section~\ref{subsec-exp-1-2}); for the problem class of maximizing weakly monotone and strongly submodular functions, we add the two applications of search-and-tracking and recommender systems (i.e., Section~\ref{subsec-exp-2}); for the problem class of maximizing DAG monotone submodular functions, we only considered the synthetic data set, and now add the real-world data set in Section~\ref{subsec-exp-3}. We also extend the discussion about the experimental results.

The rest of this paper is organized as follows. Section~\ref{sec-prob} introduces the studied problem classes, and Section~\ref{sec-moea} introduces how to apply the GSEMO to solve them. Sections~\ref{sec-theory1} to~\ref{sec-theory3} give the theoretical analyses of the GSEMO for solving the problem classes of maximizing prefix monotone submodular functions, weakly monotone and strongly submodular functions, and DAG monotone submodular functions, respectively. Section~\ref{sec-exp} presents the empirical study. Section~\ref{sec-conclusion} concludes the paper.

\section{Maximizing Monotone Submodular Functions over Sequences}\label{sec-prob}

Let $\mathbb{R}$, $\mathbb{R}_{+}$ and $\mathbb{Z}_{+}$ denote the set of reals, non-negative reals and non-negative integers, respectively. Given a finite set $V=\{v_1,v_2,\ldots,v_n\}$ of items, we study the functions $f:\mathcal{S} \rightarrow \mathbb{R}$ defined on sequences of items from $V$. A sequence is represented by $s \in \mathcal{S}=\{(s_1,s_2,\ldots,s_l) \mid s_i \in V, l \in \mathbb{Z}_{+}\}$, where $l$ is the length of the sequence. When $l=0$, it represents the empty sequence $\emptyset$. For two sequences $s, t \in \mathcal{S}$, we use $\sqsubseteq_{\mathrm{subseq}}$, $\sqsubseteq_{\mathrm{prefix}}$ and $\sqsubseteq_{\mathrm{suffix}}$ to denote their relationships. That is,
\begin{itemize}
\item $s \sqsubseteq_{\mathrm{subseq}} t$, if $s$ is a subsequence of $t$;
\item $s \sqsubseteq_{\mathrm{prefix}} t$, if $s$ is a prefix of $t$;
\item $s \sqsubseteq_{\mathrm{suffix}} t$, if $s$ is a suffix of $t$.
\end{itemize}
Note that if $s$ is a prefix or suffix of $t$, it must be a subsequence of $t$. We will use $\oplus$ to denote the concatenation of two sequences, and represent a singleton sequence $(v)$ by $v$ for simplicity.

In the following, we will first introduce the notions of monotonicity and submodularity for functions over sequences, respectively, and then introduce the problem classes studied in this paper.

\subsection{Monotonicity}

Monotonicity intuitively implies that the function value will not decrease as a sequence extends. Due to the various ways of extension, several notions of monotonicity have been introduced.

\begin{definition}[Subsequence Monotonicity~\citep{alaei2021maximizing}]\label{def-subseq-mon}
A sequence function $f:\mathcal{S} \rightarrow \mathbb{R}$ is subsequence monotone if $\;\forall s \sqsubseteq_{\mathrm{subseq}} t \in \mathcal{S}: f(s) \leq f(t)$.
\end{definition}

\begin{definition}[Prefix Monotonicity~\citep{streeter2008online,zhang2016string}]\label{def-pre-mon}
A sequence function $f:\mathcal{S} \rightarrow \mathbb{R}$ is prefix monotone if $\;\forall s \sqsubseteq_{\mathrm{prefix}} t \in \mathcal{S}: f(s) \leq f(t)$.
\end{definition}

\begin{definition}[Suffix Monotonicity~\citep{streeter2008online,zhang2016string}]\label{def-suf-mon}
A sequence function $f:\mathcal{S} \rightarrow \mathbb{R}$ is suffix monotone if $\;\forall s \sqsubseteq_{\mathrm{suffix}} t \in \mathcal{S}: f(s) \leq f(t)$.
\end{definition}

\begin{definition}[Weak Monotonicity~\citep{bernardini2020through}]\label{def-weak-mon}
A sequence function $f:\mathcal{S} \rightarrow \mathbb{R}$ is weakly monotone if $\;\forall s, t \in \mathcal{S}$, there exists $w \in \mathcal{S}$ satisfying that \begin{align}\label{cond-weak-mon}s \sqsubseteq_{\mathrm{subseq}} w, t \sqsubseteq_{\mathrm{subseq}} w, |w| \leq |s|+|t|\; \text{and} \;f(s) \leq f(w).\end{align}
\end{definition}

Because a prefix or suffix is a special case of subsequence, a subsequence monotone function must be prefix monotone and suffix monotone. For a prefix monotone function $f$, Eq.~(\refeq{cond-weak-mon}) holds by letting $w=s \oplus t$, and thus $f$ satisfies the weak monotonicity. For a suffix monotone function $f$, Eq.~(\refeq{cond-weak-mon}) holds by letting $w=t \oplus s$, and thus $f$ also satisfies the weak monotonicity. The relationship among these notions of monotonicity is shown in Figure~\ref{fig_relation-mon-and-sub}(a).

\begin{figure*}[t!]\centering
\begin{minipage}[c]{0.51\linewidth}\centering
        \includegraphics[width=1\linewidth]{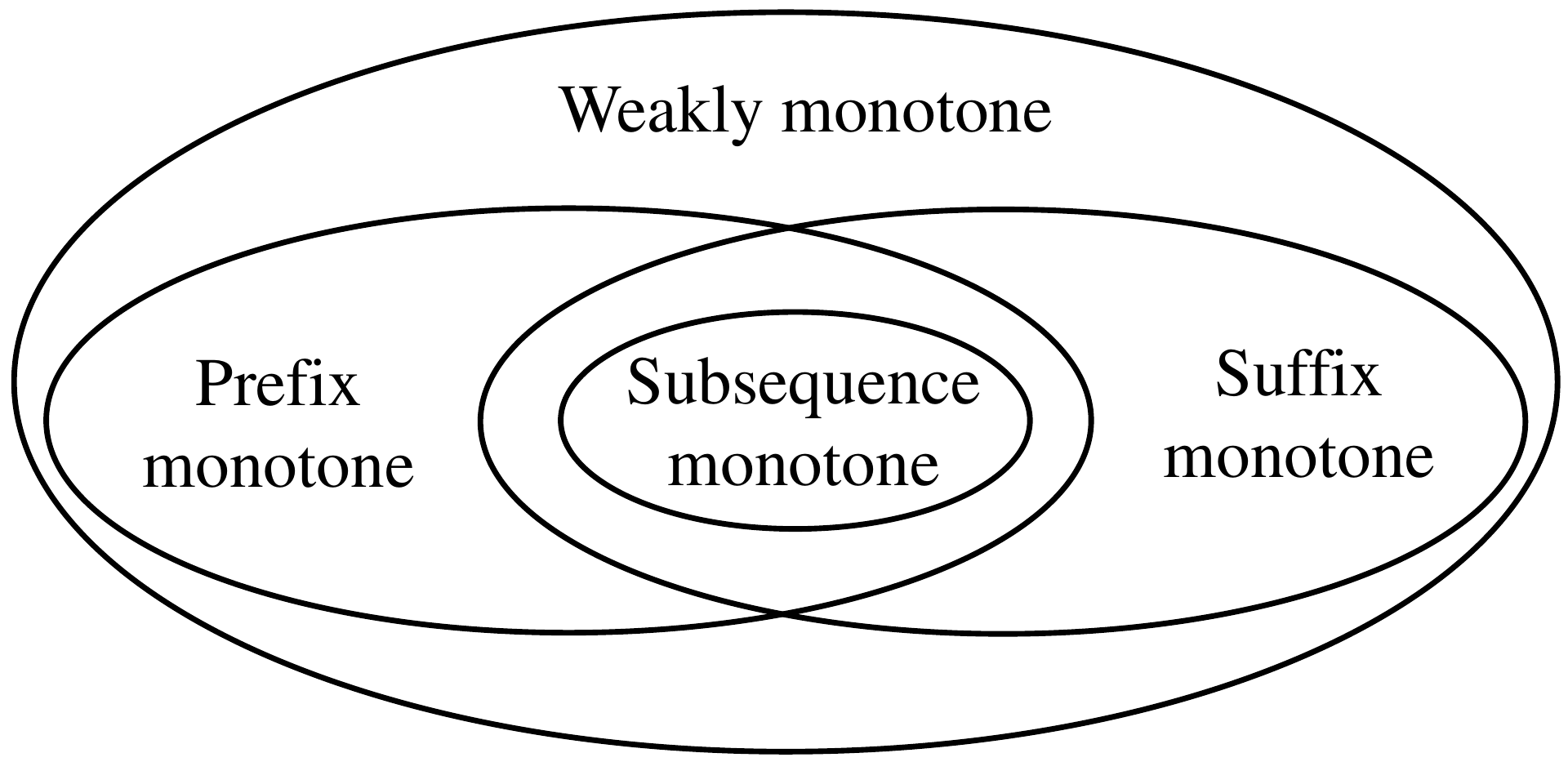}
\end{minipage}\hspace{1em}
\begin{minipage}[c]{0.45\linewidth}\centering
        \includegraphics[width=1\linewidth]{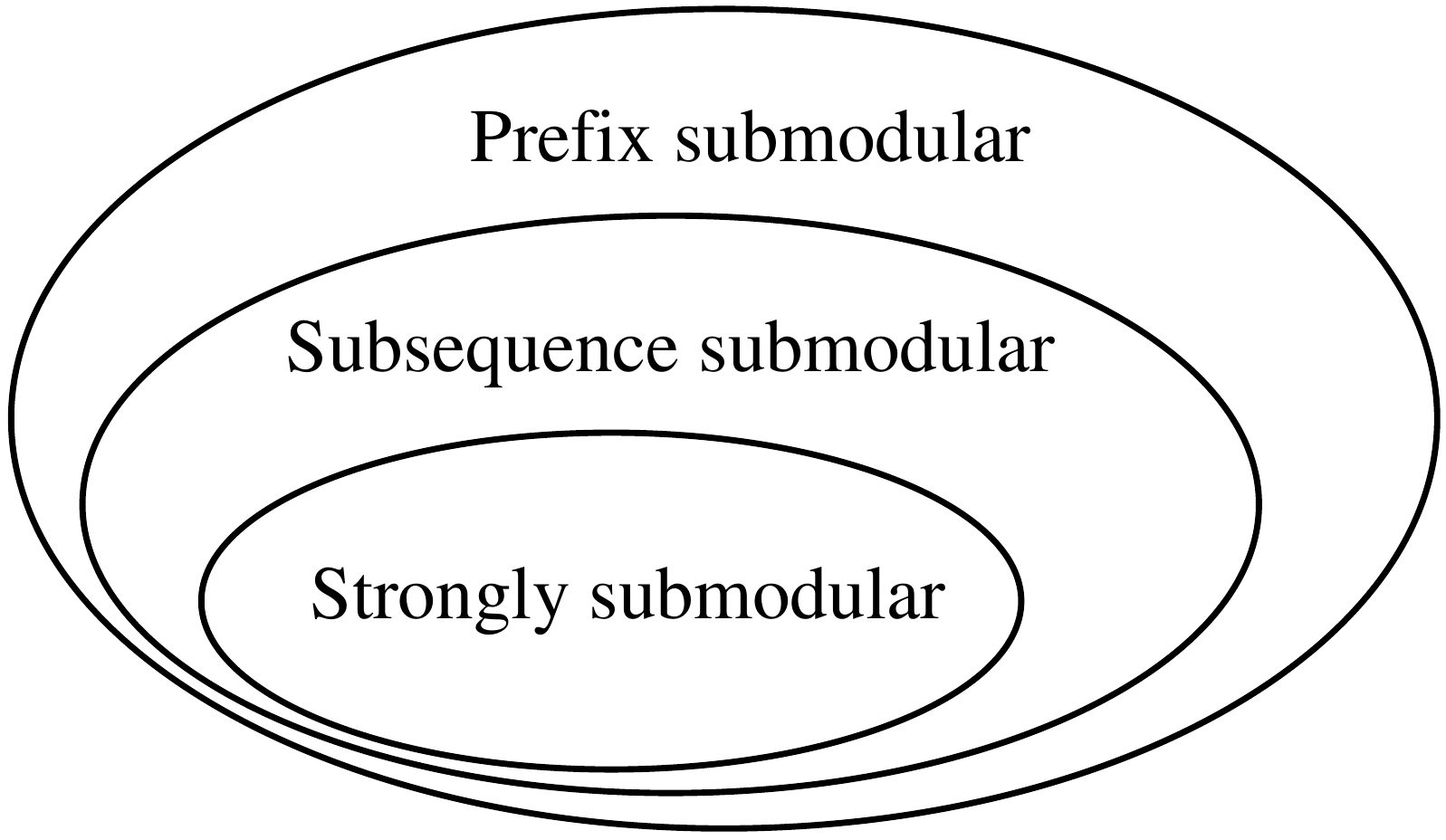}
\end{minipage}\\\vspace{0.5em}
\begin{minipage}[c]{0.51\linewidth}\centering
    \small(a) Monotonicity
\end{minipage}\hspace{1em}
\begin{minipage}[c]{0.45\linewidth}\centering
    \small(b) Submodularity
\end{minipage}
\caption{The relationship among different properties of sequence functions: (a) monotonicity and (b) submodularity.}\label{fig_relation-mon-and-sub}
\end{figure*}

For a function $f$ with any kind of monotonicity in Definitions~\ref{def-subseq-mon}--\ref{def-weak-mon}, the empty sequence $\emptyset$ has the smallest $f$ value. In this paper, we always assume that monotone functions are normalized, i.e., $f(\emptyset)=0$. Thus, it holds that $\forall s \in \mathcal{S}: f(s) \geq 0$.

\subsection{Submodularity}

Submodularity intuitively implies the diminishing returns property, i.e., the benefit of inserting an item into a sequence will not increase as the sequence extends. Due to the different ways of insertion and extension, several notions of submodularity have been introduced.

\begin{definition}[Strong Submodularity~\citep{bernardini2020through}]\label{def-strong-sub}
A sequence function $f:\mathcal{S} \rightarrow \mathbb{R}$ is strongly submodular if
\begin{align}\label{cond-strong-sub}\forall s \sqsubseteq_{\mathrm{subseq}} t\in \mathcal{S}, o \in \mathcal{S}, v \in V: f(s \oplus v \oplus o)-f(s \oplus o) \geq f(t \oplus v \oplus o)-f(t \oplus o).\end{align}
\end{definition}

\begin{definition}[Subsequence Submodularity~\citep{alaei2021maximizing}]\label{def-subseq-sub}
A sequence function $f:\mathcal{S} \rightarrow \mathbb{R}$ is subsequence submodular if
\begin{align}\label{cond-subseq-sub}\forall s \sqsubseteq_{\mathrm{subseq}} t\in \mathcal{S}, v \in V: f(s \oplus v)-f(s) \geq f(t \oplus v)-f(t).\end{align}
\end{definition}

\begin{definition}[Prefix Submodularity~\citep{streeter2008online,zhang2016string}]\label{def-pre-sub}
A sequence function $f:\mathcal{S} \rightarrow \mathbb{R}$ is prefix submodular if
\begin{align}\forall s \sqsubseteq_{\mathrm{prefix}} t\in \mathcal{S}, v \in V: f(s \oplus v)-f(s) \geq f(t \oplus v)-f(t).\end{align}
\end{definition}

By letting $o=\emptyset$, Eq.~(\refeq{cond-strong-sub}) implies Eq.~(\refeq{cond-subseq-sub}), and thus a strongly submodular function must be subsequence submodular. Because a prefix must be a subsequence, a subsequence submodular function must satisfy the prefix submodularity. Their relationship is shown in Figure~\ref{fig_relation-mon-and-sub}(b).

\begin{definition}[Curvature~\citep{zhang2016string}]\label{def-strcur}
The curvature of a sequence function $f: \mathcal{S} \rightarrow \mathbb{R}$ w.r.t. a sequence $s$ and a parameter $m \geq 1$ is
\begin{align}\label{eq:curvature}
\sigma_{s,m}(f)=\max_{t \in \mathcal{S}, 0<|t|\leq m} \left\{1-\frac{f(t \oplus s)-f(s)}{f(t)}\right\}.
\end{align}
\end{definition}

For a monotone submodular function, its curvature characterizes the degree of submodularity. When $f$ is clear, we will use $\sigma_{s,m}$ for short. We then make the following observations:
\begin{remark}\label{remark1}
For any monotone submodular function $f: \mathcal{S} \rightarrow \mathbb{R}_{+}$, it holds that\vspace{-1em}
\begin{itemize}
\item[(1)] $\sigma_{s,m}$ increases with $m$ for any $s$;
\item[(2)] $\sigma_{s,m} \geq 0$ for any $s$ and $m \geq |s|-1$, because by letting $t$ in Eq.~(\refeq{eq:curvature}) be $(s_1,\ldots,s_{|s|-1})$, we have
\begin{align}
f(t \oplus s)-f(t) &=\sum\limits^{|s|}_{i=1} f(t \oplus (s_1,\ldots,s_i)) -f(t \oplus (s_1,\ldots,s_{i-1}))\\
&\leq \sum\limits^{|s|}_{i=1} f((s_1,\ldots,s_i)) -f((s_1,\ldots,s_{i-1}))=f(s),
\end{align}
where the inequality holds by the submodularity (which can be any one in Definitions~\ref{def-strong-sub}--\ref{def-pre-sub}) of~$f$ since for any $1 \leq i \leq |s|$, $(s_1,\ldots,s_{i-1}) \sqsubseteq_{\mathrm{prefix}} t \oplus (s_1,\ldots,s_{i-1})=(s_1,\ldots,s_{|s|-1},s_1,\ldots,s_{i-1})$.
\end{itemize}
\end{remark}

\subsection{Problem Classes}\label{sec-problemclass}

For monotone submodular maximization over sequences, we will study the problem of Eq.~(\refeq{eq-sequence}), which is also the only one studied so far. Given all items $V=\{v_1,v_2,\ldots,v_n\}$, a monotone submodular function $f: \mathcal{S} \rightarrow \mathbb{R}$ and a budget $k \in \mathbb{Z}_+$, the goal is to select a sequence $s$ such that $f$ is maximized with the constraint $|s|\leq k$. Previous studies~\citep{alaei2021maximizing,bernardini2020through,streeter2008online,tschiatschek2017selecting,zhang2016string} considered the problems with different combinations of monotonicity and submodularity, as presented in Definitions~\ref{def:prob-subsequence}--\ref{def:prob-weak} and~\ref{def:prob-dag}.

In~\citep{alaei2021maximizing}, the problem when the objective function $f$ satisfies the subsequence monotonicity in Definition~\ref{def-subseq-mon} and subsequence submodularity in Definition~\ref{def-subseq-sub} was considered, as presented in Definition~\ref{def:prob-subsequence}. It has been proved that the greedy algorithm can achieve a $(1-1/e)$-approximation guarantee~\citep{alaei2021maximizing}, i.e., the output sequence $s$ satisfies $f(s) \geq (1-1/e) \cdot \mathrm{OPT}$, where $\mathrm{OPT}$ denotes the optimal function value. As described in Algorithm~\ref{alg:greedy}, the greedy algorithm iteratively appends one item with the largest improvement on $f$ to the end of the current sequence.

\begin{definition}[Maximizing Subsequence Monotone Submodular Functions~\citep{alaei2021maximizing}]\label{def:prob-subsequence}
Given a subsequence monotone and subsequence submodular function $f: \mathcal{S} \rightarrow \mathbb{R}_{+}$ and a budget $k\in \mathbb{Z}_+$, to find a sequence of at most $k$ items maximizing $f$, i.e.,
\begin{align}\label{eq:prob-subsequence}
\mathop{\arg\max}\nolimits_{s\in \mathcal{S}} f(s) \quad \text{s.t.}\quad |s|\leq k.
\end{align}
\end{definition}

\begin{algorithm}[h!]\caption{Greedy Algorithm~\citep{alaei2021maximizing}}
\label{alg:greedy}
\textbf{Input}: all items $V=\{v_1,v_2,\ldots,v_n\}$, a monotone submodular function $f:\mathcal{S} \rightarrow \mathbb{R}_+$ and a budget $k$\\
\textbf{Output}: a sequence $s \in \mathcal{S}$ with $|s| = k$\\
\textbf{Process}:
\begin{algorithmic}[1]
\STATE Let $s=\emptyset$ and $i=0$;
\STATE \textbf{repeat}
\STATE \quad $v^*=\arg\max_{v \in V} f(s \oplus v)$;
\STATE \quad $s=s \oplus v^*$, and $i=i+1$
\STATE \textbf{until} $i=k$
\STATE \textbf{return} $s$
\end{algorithmic}
\end{algorithm}

In~\citep{streeter2008online}, a more general problem was considered by relaxing the monotonicity and submodularity of~$f$. As presented in Definition~\ref{def:prob-presuf}, $f$ is required to satisfy the prefix monotonicity in Definition~\ref{def-pre-mon}, suffix monotonicity in Definition~\ref{def-suf-mon} and prefix submodularity in Definition~\ref{def-pre-sub}. As shown in Figure~\ref{fig_relation-mon-and-sub}, both prefix and suffix monotonicity are weaker than subsequence monotonicity, and prefix submodularity is weaker than subsequence submodularity. This implies that the problem in Definition~\ref{def:prob-presuf} is more general than that in Definition~\ref{def:prob-subsequence}. It has been proved that the greedy algorithm in Algorithm~\ref{alg:greedy} can still achieve the $(1-1/e)$-approximation guarantee~\citep{streeter2008online}.

\begin{definition}[Maximizing Prefix$+$Suffix Monotone and Prefix Submodular Functions~\citep{streeter2008online}]\label{def:prob-presuf}
Given a prefix monotone, suffix monotone and prefix submodular function $f: \mathcal{S} \rightarrow \mathbb{R}_{+}$ and a budget $k \in \mathbb{Z}_+$, to find a sequence of at most $k$ items maximizing $f$, i.e.,
\begin{align}\label{eq:prob-presuf}
\mathop{\arg\max}\nolimits_{s\in \mathcal{S}} f(s) \quad \text{s.t.}\quad |s|\leq k.
\end{align}
\end{definition}

In~\citep{zhang2016string}, the monotonicity of the objective function $f$ is further relaxed. As presented in Definition~\ref{def:prob-pre}, $f$ is required to satisfy the prefix monotonicity and prefix submodularity, while the suffix monotonicity is not necessarily satisfied. Thus, this problem is more general than that in Definition~\ref{def:prob-presuf}. The left subfigure in Figure~\ref{fig_relation-mon-sub} shows the relationship among the problems in Definitions~\ref{def:prob-subsequence}--\ref{def:prob-pre}. By utilizing the notion of curvature in Definition~\ref{def-strcur}, it has been proved that the greedy algorithm achieves an approximation guarantee of $(1/\sigma_{o,k})(1-e^{-\sigma_{o,k}})$~\citep{zhang2016string}, where $o$ denotes an optimal sequence of Eq.~(\refeq{eq:prob-pre}), i.e., $f(o)=\mathrm{OPT}$.

\begin{definition}[Maximizing Prefix Monotone Submodular Functions~\citep{zhang2016string}]\label{def:prob-pre}
Given a prefix monotone and prefix submodular function $f: \mathcal{S} \rightarrow \mathbb{R}_{+}$ and a budget $k\in \mathbb{Z}_+$, to find a sequence of at most $k$ items maximizing $f$, i.e.,
\begin{align}\label{eq:prob-pre}
\mathop{\arg\max}\nolimits_{s\in \mathcal{S}} f(s) \quad \text{s.t.}\quad |s|\leq k.
\end{align}
\end{definition}

In~\citep{bernardini2020through}, the problem when the objective function $f$ satisfies the weak monotonicity in Definition~\ref{def-weak-mon} and strong submodularity in Definition~\ref{def-strong-sub} was considered. Compared with the problem in Definition~\ref{def:prob-pre}, the monotonicity and submodularity of $f$ are relaxed and strengthened, respectively, because weak monotonicity is weaker than prefix monotonicity while strong submodularity is stronger than prefix submodularity, as shown in Figure~\ref{fig_relation-mon-and-sub}. It has been proved that the greedy algorithm fails to achieve a constant approximation guarantee, while the generalized greedy algorithm can achieve an approximation guarantee of $1-1/e$~\citep{bernardini2020through}. As presented in Algorithm~\ref{alg:general-greedy}, the generalized greedy algorithm iteratively augments the sequence by inserting an item with the largest marginal gain on $f$. Note that the greedy algorithm in Algorithm~\ref{alg:greedy} can only insert a new item into the end of the sequence, while the generalized greedy algorithm allows a new item to be inserted into any position of the sequence.

\begin{definition}[Maximizing Weakly Monotone and Strongly Submodular Functions~\citep{bernardini2020through}]\label{def:prob-weak}
Given a weakly monotone and strongly submodular function $f: \mathcal{S} \rightarrow \mathbb{R}_{+}$ and a budget $k\in \mathbb{Z}_+$, to find a sequence of at most $k$ items maximizing $f$, i.e.,
\begin{align}\label{eq:prob-weak}
\mathop{\arg\max}\nolimits_{s\in \mathcal{S}} f(s) \quad \text{s.t.}\quad |s|\leq k.
\end{align}
\end{definition}

\begin{algorithm}[h!]\caption{Generalized Greedy Algorithm~\citep{bernardini2020through}}
\label{alg:general-greedy}
\textbf{Input}: all items $V=\{v_1,v_2,\ldots,v_n\}$, a weakly monotone and strongly submodular function $f:\mathcal{S} \rightarrow \mathbb{R}_+$ and a budget $k$\\
\textbf{Output}: a sequence $s \in \mathcal{S}$ with $|s| = k$\\
\textbf{Process}:
\begin{algorithmic}[1]
\STATE Let $s=\emptyset$ and $i=0$;
\STATE \textbf{repeat}
\STATE \quad $s'=\arg\max_{s': s \sqsubseteq_{\mathrm{subseq}} s', |s'|=|s|+1} f(s')$;
\STATE \quad $s=s'$, and $i=i+1$
\STATE \textbf{until} $i=k$
\STATE \textbf{return} $s$
\end{algorithmic}
\end{algorithm}

In~\citep{tschiatschek2017selecting}, a special class of objective functions, so-called DAG monotone submodular functions, was considered. As presented in Definition~\ref{def:dagms}, the items in a sequence cannot be repeated here, i.e., $\mathcal{S}=\{(s_1,s_2,\ldots,s_l) \mid s_i \in V, l \in \mathbb{Z}_{+}, \forall i\neq j: s_i \neq s_j\}$. It is assumed that there exists a directed acyclic graph $G=(V,E)$ (if not counting self-cycles) capturing the ordered preferences among items, where an edge $(v_i,v_j) \in E$ means that there is an additional utility when selecting $v_i$ before $v_j$; and the function $f$ value of a sequence $s$ can be determined by the set $E(s)$ of edges induced by $s$ and a corresponding monotone submodular set function $h: 2^{E} \rightarrow \mathbb{R}_+$. Note that for a set function $h: 2^{E} \rightarrow \mathbb{R}_+$, monotonicity implies $\forall X \subseteq Y \subseteq E: h(Y) \geq h(X)$, and submodularity implies $\forall X \subseteq Y \subseteq E, v \notin Y: h(X \cup \{v\})-h(X) \geq h(Y \cup \{v\}) - h(Y)$. A DAG monotone submodular function must satisfy the subsequence monotonicity, because $\forall s \sqsubseteq_{\mathrm{subseq}} t \in \mathcal{S}$, their sets of induced edges must satisfy $E(s) \subseteq E(t)$, and thus $h(E(s))\leq h(E(t))$, implying $f(s) \leq f(t)$. It has been shown that a DAG monotone submodular function does not necessarily satisfy the prefix submodular property~\citep{tschiatschek2017selecting}. The right subfigure in Figure~\ref{fig_relation-mon-sub} shows the relationship among the problems in Definitions~\ref{def:prob-pre},~\ref{def:prob-weak} and~\ref{def:prob-dag}. Note that the intersection of the objective function classes in Definitions~\ref{def:prob-pre} and~\ref{def:prob-weak} satisfies the prefix monotonicity and strong submodularity, and the intersection of that in Definitions~\ref{def:prob-weak} and~\ref{def:prob-dag} must be a subset of the former one, because a DAG monotone submodular function must be subsequence monotone and also prefix monotone.

\begin{definition}[DAG Monotone Submodularity]\label{def:dagms}
Given a directed acyclic graph (DAG) $G=(V,E)$ (if not counting self-cycles) modeling the ordered preferences among items, and $\mathcal{S}=\{(s_1,s_2,\ldots,s_l) \mid s_i \in V, l \in \mathbb{Z}_{+}, \forall i\neq j: s_i \neq s_j\}$, a sequence function $f:\mathcal{S} \rightarrow \mathbb{R}$ is DAG monotone submodular if \begin{align}\forall s \in \mathcal{S}: f(s)=h(E(s)),\end{align} where $h:2^{E} \rightarrow \mathbb{R}_{+}$ is a monotone submodular set function, and $E(s)=\{(s_i,s_j) \mid (s_i,s_j)\in E, i \leq j\}$ is the set of edges induced by the sequence $s$.
\end{definition}

\begin{definition}[Maximizing DAG Monotone Submodular Functions~\citep{tschiatschek2017selecting}]\label{def:prob-dag}
Given a DAG monotone submodular function $f: \mathcal{S} \rightarrow \mathbb{R}_{+}$ and a budget $k\in \mathbb{Z}_+$, to find a sequence of at most $k$ items maximizing $f$, i.e.,
\begin{align}\label{eq:prob-dag}
\mathop{\arg\max}\nolimits_{s\in \mathcal{S}} f(s) \quad \text{s.t.}\quad |s|\leq k.
\end{align}
\end{definition}

For the above problem, the greedy algorithm in Algorithm~\ref{alg:greedy} also fails to achieve a constant approximation guarantee~\citep{tschiatschek2017selecting}. Tschiatschek et al.~\cite{tschiatschek2017selecting} then developed the \textsc{OMegA} algorithm by exploiting the DAG property of the graph $G$, i.e., for each set of items, its optimal ordering can be computed by first computing a topological ordering of $G$ and then sorting the set of items according to that order. Let $V(Q)$ denote the set of items present in a set $Q$ of edges, and let $\textsc{Reorder}(\cdot)$ denote the optimal sequence for a set of items. As presented in Algorithm~\ref{alg:OMegA}, the \textsc{OMegA} algorithm greedily picks an edge instead of an item. It has been proved that the \textsc{OMegA} can achieve an approximation guarantee of $1-e^{-1/(2\Delta)}$~\citep{tschiatschek2017selecting}, where $\Delta=\min\{\Delta_{\text{in}}, \Delta_{\text{out}}\}$, and $\Delta_{\text{in}}$, $\Delta_{\text{out}}$ are the largest indegree and outdegree of the items in the graph $G$, respectively.

\begin{algorithm}[h!]\caption{\textsc{OMegA} Algorithm~\citep{tschiatschek2017selecting}}
\label{alg:OMegA}
\textbf{Input}: a directed acyclic graph $G=(V,E)$, a DAG monotone submodular function $f:\mathcal{S} \rightarrow \mathbb{R}_+$ and a budget $k$\\
\textbf{Output}: a sequence $s \in \mathcal{S}$ with $|s| \leq k$\\
\textbf{Process}:
\begin{algorithmic}[1]
\STATE Let $Q=\emptyset$;
\WHILE {$\exists e \in E\setminus Q$ such that $|V(Q \cup \{e\})| \leq k$}
\STATE \quad $C=\{e \in E\setminus Q \mid |V(Q \cup \{e\})| \leq k\}$;
\STATE \quad $e^*=\arg\max_{e \in C} f(\textsc{Reorder}(V(Q \cup \{e\})))$;
\STATE \quad $Q=Q \cup \{e^*\}$
\ENDWHILE
\STATE $s=\textsc{Reorder}(V(Q))$
\STATE \textbf{return} $s$
\end{algorithmic}
\end{algorithm}

Note that the size of the search space is $\binom{n}{k}$ for selecting a subset with $k$ items from a total set of $n$ items, while by considering sequences instead of subsets, the search space $\mathcal{S}$ becomes much larger. The size is $n^k$ if allowing repeated items in a sequence, and is $k! \cdot \binom{n}{k}$ if not, either of which is exponentially larger w.r.t. $k$.

\section{Multi-objective Evolutionary Algorithms}\label{sec-moea}

To examine the approximation performance of EAs solving the problem classes of monotone submodular maximization over sequences, we consider a simple multi-objective EA, i.e., GSEMO. Compared with the original version~\citep{Laumanns04}, the GSEMO here starts from the empty sequence instead of a randomly chosen one, and updates the mutation operator accordingly to be suitable for sequences.

The original problem Eq.~(\refeq{eq-sequence}) is reformulated as
\begin{equation}
\begin{aligned}\label{def-CO-BO}
\arg\max\nolimits_{s \in \mathcal{S}}& \quad  \big(f_1(s),\;f_2(s)\big),
\end{aligned}
\end{equation}
\begin{equation}
\text{where } \begin{aligned}
f_1(s) = \begin{cases}
	-\infty, & |s| \geq 2k\\
	f(s), &\text{otherwise}
\end{cases},\quad
f_2(s) =-|s|.
\end{aligned}
\end{equation}
That is, the objective function $f$ is to be maximized and meanwhile the sequence length $|s|$ is to be minimized. The goal of setting $f_1$ to $-\infty$ for $|s| \geq 2k$ is to exclude overly infeasible sequences, the length of which is at least $2k$. In fact, only feasible sequences are needed in theoretical analyses. That is, it is sufficient to set $f_1$ to $-\infty$ for $|s| > k$. The reason of setting $f_1$ to $-\infty$ for $|s| \geq 2k$ is to allow infeasible sequences (i.e., sequences with $k<|s|<2k$) with small constraint violation degree to participate in the evolutionary process, which may bring performance improvement in practice. We will compare these two settings in the experiments.

In the bi-objective setting, both the two objective values have to be considered for comparing two sequences $s$ and $s'$, and the domination relationship in Definition~\ref{def_Domination} is often used.
\begin{definition}[Domination]\label{def_Domination}
For two sequences $s$ and $s'$,\vspace{-1em}
\begin{itemize}
  \item $s$ \emph{weakly dominates} $s'$ (i.e., $s$ is \emph{better} than $s'$, denoted as $s \succeq s'$) if $f_1(s) \geq f_1(s')$ and $f_2(s) \geq f_2(s')$;
  \item $s$ \emph{dominates} $s'$ (i.e., $s$ is \emph{strictly better} than $s'$, denoted as $s \succ s'$) if $s\succeq s'$ and either $f_1(s) > f_1(s')$ or $f_2(s) > f_2(s')$.
\end{itemize}\vspace{-1em}
\end{definition}
If neither $s$ weakly dominates $s'$ (i.e., $s \succeq s'$) nor $s'$ weakly dominates $s$ (i.e., $s' \succeq s$), $s$ and $s'$ are \emph{incomparable}.

The procedure of GSEMO is described in Algorithm~\ref{alg:POSeqSel}. Starting from the empty sequence $\emptyset$ (line~1), it iteratively tries to improve the quality of the sequences in the population $P$ (lines~2--8). In each iteration, a new sequence $s'$ is generated by applying the mutation operator to an archived sequence $s$, which is randomly selected from the current $P$ (lines~3--4); if $s'$ is not dominated by any archived sequence (line~5), it will be added into $P$, and meanwhile those archived sequences weakly dominated by $s'$ will be removed (line~6). Note that the domination-based comparison makes the population $P$ always contain incomparable sequences.

\begin{algorithm}[ht!]
\caption{GSEMO Algorithm}
\label{alg:POSeqSel}
\textbf{Input}: all items $V=\{v_1,v_2,\ldots,v_n\}$, a monotone submodular function $f:\mathcal{S} \rightarrow \mathbb{R}$ and a budget $k$\\
\textbf{Output}: a sequence $s \in \mathcal{S}$ with $|s| \leq k$\\
\textbf{Process}:
\begin{algorithmic}[1]
\STATE Let $P=\{\emptyset\}$;
\WHILE{stopping criterion not satisfied}
\STATE Select a sequence $s$ from $P$ uniformly at random;
\STATE Generate a new sequence $s'$ by applying mutation in Definition~\ref{def:mutate} to $s$;
\STATE \textbf{if} $\nexists\, t \in P$ such that $t \succ s'$ \textbf{then}
\STATE \quad $P = (P \setminus  \{t \in P \mid s' \succeq t\})\cup\{s'\}$
\STATE \textbf{end if}
\ENDWHILE
\STATE \textbf{return} $\arg\max_{s\in P: |s| \leq k} f(s)$
\end{algorithmic}
\end{algorithm}

To generate a new sequence $s'$ from $s$ in line~4 of Algorithm~\ref{alg:POSeqSel}, the mutation operator in Definition~\ref{def:mutate} applies the insertion or deletion operator uniformly at random (i.e., each with probability $1/2$) and repeats this process $r$ times independently, where the number $r$ is determined by the Poisson distribution with parameter $\lambda=1$. Note that the way of using a Poisson distributed random variable to determine the number of local operators performed by mutation has been used for permutation and tree spaces~\citep{durrett2011computational,scharnow2005analysis}. The insertion operator in Definition~\ref{def:insert} inserts a randomly chosen item into a randomly chosen position of the sequence. The deletion operator in Definition~\ref{def:delete} deletes a randomly chosen item of the sequence. Note that for the empty sequence, the deletion operator will keep it unchanged; when not allowing repeated items, the insertion operator will keep a sequence of length $n$ (i.e., the maximum length) unchanged.

\begin{definition}[Mutation]\label{def:mutate}
Given a sequence $s \in \mathcal{S}$, apply the insertion in Definition~\ref{def:insert} or deletion in Definition~\ref{def:delete} operator uniformly at random, and repeat this process independently $r$ times, where $r$ is a random number sampled from the Poisson distribution with parameter $\lambda=1$.
\end{definition}

\begin{definition}[Insertion]\label{def:insert}
Given a sequence $s \in \mathcal{S}$, if allowing repeated items, the insertion operator first randomly selects an item $v$ from $V$ and then inserts $v$ into a randomly chosen position of $s$. That is, a new sequence $(s_1,\ldots,s_{i-1},v,s_{i},\ldots,s_{|s|})$ is generated, where $v$ is uniformly chosen from $V$ at random, and $i$ is uniformly chosen from $\{1,2,\ldots,|s|+1\}$ at random. If not allowing repeated items, the item $v$ is randomly selected from $V \setminus V(s)$, where $V(s)$ denotes the set of items appearing in the sequence $s$.
\end{definition}

\begin{definition}[Deletion]\label{def:delete}
Given a sequence $s \in \mathcal{S}$, the deletion operator deletes a randomly chosen item of $s$. That is, a new sequence $(s_1,\ldots,s_{i-1},s_{i+1},\ldots,s_{|s|})$ is generated, where $i$ is uniformly chosen from $\{1,2,\ldots,|s|\}$ at random.
\end{definition}

The GSEMO~\citep{Laumanns04} was originally used for optimizing pseudo-Boolean functions, and employs the bit-wise mutation operator to generate new solutions. The bit-wise mutation operator flips each bit of a Boolean vector with probability $1/n$, which is equivalent to first selecting a number $r$ randomly from the binomial distribution $\mathcal{B}(n,1/n)$ and then flipping $r$ randomly chosen bits of the Boolean vector. Note that the Poisson distribution with parameter $\lambda=1$ is an approximation to the binomial distribution $\mathcal{B}(n,1/n)$. Furthermore, flipping a bit value from 0 to 1 and from 1 to 0 can be viewed as inserting and deleting an item, respectively. Thus, the mutation operator in Definition~\ref{def:mutate} is actually a natural extension of the widely used bit-wise mutation operator over Boolean vectors.

When the GSEMO terminates, it will output the best sequence (i.e., having the largest $f$ value) satisfying the length constraint in the population $P$ (line~9). Thus, we can see that bi-objective optimization here is used as an intermediate step to solve the original single-objective optimization problem, and what we really concern is the quality of the best sequence w.r.t. the original single-objective problem, in the population found by the GSEMO, rather than the quality of the population w.r.t. the transformed bi-objective problem.

The number $T$ of iterations run by the GSEMO could affect the quality of the produced sequence. In the following three sections, we will analyze the approximation guarantee as well as the required number of iterations for the GSEMO solving the problem classes in Definitions~\ref{def:prob-subsequence}--\ref{def:prob-weak} and~\ref{def:prob-dag}. Note that we will use the theoretically derived value of $T$ to run the GSEMO in the experiments.

\section{Theoretical Analysis on Maximizing Prefix Monotone Submodular Functions}\label{sec-theory1}

Let $\mathrm{OPT}$ denote the optimal function value, and let $\mathbb{E}[T]$ denote the expected number of iterations required by the GSEMO to achieve a desired approximation guarantee for the first time. First, we prove in Theorem~\ref{theo-prefix} that for the problem class of maximizing prefix monotone submodular functions in Definition~\ref{def:prob-pre}, the GSEMO with $\mathbb{E}[T] \leq 2ek^2(k+1)n$ can achieve a $(1/\sigma_{o,k-1})(1-e^{-\sigma_{o,k-1}})$-approximation guarantee, which reaches the best known one previously obtained by the greedy algorithm~\citep{zhang2016string}. Note that $o$ denotes an optimal sequence of Eq.~(\refeq{eq:prob-pre}), and $\sigma_{o,k-1} \geq 0$ by Remark~\ref{remark1} since $|o|\leq k$.

The proof of Theorem~\ref{theo-prefix} relies on Lemma~\ref{lem-prob2-onestep}, i.e., for any sequence $s$, there always exists one item whose addition to the end of $s$ can bring an increment on $f$, roughly proportional to the current distance to the optimum.

\begin{lemma}\label{lem-prob2-onestep}
Let $f:\mathcal{S} \rightarrow \mathbb{R}_+$ be a prefix monotone and prefix submodular function. For any sequence $s \in \mathcal{S}$, there exists one item $v \in V$ such that \begin{align}f(s \oplus v)-f(s) \geq (\mathrm{OPT}-\sigma_{o,|s|} \cdot f(s))/k,\end{align}
where $o$ denotes an optimal sequence of Eq.~(\refeq{eq:prob-pre}), i.e., $f(o)=\mathrm{OPT}$, and $\sigma_{o,|s|}$ is the curvature in Definition~\ref{def-strcur}.
\end{lemma}
\begin{proof}
For any $s \in \mathcal{S}$, we have
\begin{align}\label{eq:mid1}
f(s \oplus o)-f(s)&=\sum\limits^{|o|}_{i=1} (f(s \oplus (o_1,\ldots,o_i))-f(s \oplus (o_1,\ldots,o_{i-1})))\\
&\leq \sum\limits^{|o|}_{i=1} (f(s \oplus o_i)-f(s)),
\end{align}
where the inequality holds by $s \sqsubseteq_{\mathrm{prefix}} s \oplus (o_1,\ldots,o_{i-1})$ and the prefix submodularity of $f$ (i.e., Definition~\ref{def-pre-sub}). Let $V(o)$ denote the set of items appearing in the sequence $o$, and $v^* \in \arg\max_{v \in V(o)} f(s \oplus v)$. Then, we have $f(s \oplus o)-f(s) \leq |o|\cdot (f(s \oplus v^*)-f(s))$, implying
\begin{align}\label{eq:mid2}
f(s \oplus v^*)-f(s) &\geq (f(s \oplus o)-f(s))/|o|\\
&\geq (f(s \oplus o)-f(s))/k,
\end{align}
where the last inequality holds by $|o| \leq k$ and $f(s \oplus o) -f(s) \geq 0$. Note that $f(s \oplus o) -f(s) \geq 0$ holds because $f$ is prefix monotone and $s\sqsubseteq_{\mathrm{prefix}} s \oplus o$.

By the definition of curvature (i.e., Definition~\ref{def-strcur}), we get
\begin{align}
\sigma_{o,|s|}=\max_{t \in \mathcal{S}, 0<|t|\leq |s|} \left\{1-\frac{f(t \oplus o)-f(o)}{f(t)}\right\}\geq 1-\frac{f(s \oplus o) -f(o)}{f(s)},\end{align}
leading to \begin{align}\label{eq:mid6}f(s \oplus o) \geq f(o)+(1-\sigma_{o,|s|})\cdot f(s)=\mathrm{OPT}+(1-\sigma_{o,|s|})\cdot f(s).\end{align} By applying Eq.~(\refeq{eq:mid6}) to Eq.~(\refeq{eq:mid2}), the lemma holds. Note that $\sigma_{o,|s|}$ is not defined for $s=\emptyset$, but the lemma still holds by applying $f(\emptyset)=0$ to Eq.~(\ref{eq:mid2}).
\end{proof}

\begin{theorem}\label{theo-prefix}
For the problem class of maximizing prefix monotone submodular functions in Definition~\ref{def:prob-pre}, the GSEMO with $\mathbb{E}[T] \leq 2ek^2(k+1)n$ finds a sequence $s \in \mathcal{S}$ with $|s| \leq k$ and \begin{align}f(s) \geq  (1/\sigma_{o,k-1})(1-e^{-\sigma_{o,k-1}})\cdot \mathrm{OPT},\end{align}
where $o$ denotes an optimal sequence of Eq.~(\refeq{eq:prob-pre}), i.e., $f(o)=\mathrm{OPT}$, and $\sigma_{o,k-1}$ is the curvature in Definition~\ref{def-strcur}.
\end{theorem}
\begin{proof}
Let $J_{\max}$ denote the maximum value of $j \in \{0,1,\ldots,k\}$ such that in the population $P$ of GSEMO (i.e., Algorithm~\ref{alg:POSeqSel}), there exists a sequence $s$ with $|s| \leq j$ and $f(s) \geq \frac{1}{\sigma_{o,k-1}}(1-(1-\frac{\sigma_{o,k-1}}{k})^j) \cdot \mathrm{OPT}$. That is,
\begin{align}
&J_{\max}=\max\left\{j \in \{0,1,\ldots,k\} \mid\right.\\
&\qquad \qquad \qquad \left.\exists s \in P, \; |s| \leq j \wedge f(s) \geq \frac{1}{\sigma_{o,k-1}}\left(1-\left(1-\frac{\sigma_{o,k-1}}{k}\right)^j\right) \cdot \mathrm{OPT}\right\}.
\end{align}
We then only need to analyze the expected number of iterations until $J_{\max}=k$, because it implies that there exists one sequence $s$ in $P$ satisfying that $|s| \leq k$ and \begin{align}f(s) \geq \frac{1}{\sigma_{o,k-1}}\left(1-\left(1-\frac{\sigma_{o,k-1}}{k}\right)^k\right) \cdot \mathrm{OPT} \geq \frac{1}{\sigma_{o,k-1}}(1-e^{-\sigma_{o,k-1}}) \cdot \mathrm{OPT}.\end{align}

The initial value of $J_{\max}$ is 0, since the GSEMO starts from the empty sequence $\emptyset$. Assume that currently $J_{\max}=i <k$. Let $s$ be a corresponding sequence with the value $i$, i.e., $|s|\leq i$ and
\begin{align}\label{eq:inductive}
f(s)\geq \frac{1}{\sigma_{o,k-1}}\left(1-\left(1-\frac{\sigma_{o,k-1}}{k}\right)^i\right) \cdot \mathrm{OPT}.
\end{align}
It is easy to see that $J_{\max}$ cannot decrease, because deleting $s$ from $P$ (line~6 of Algorithm~\ref{alg:POSeqSel}) implies that $s$ is weakly dominated by the newly generated sequence $s'$, which must satisfy that $|s'| \leq |s|$ and $f(s') \geq f(s)$. By Lemma~\ref{lem-prob2-onestep}, we know that appending a specific item to the end of $s$ can generate a new sequence $s'$ with
\begin{align}\label{eq:mid7}
f(s')-f(s) \geq \frac{1}{k} \cdot (\mathrm{OPT}-\sigma_{o,|s|}\cdot f(s)),
\end{align}
implying
\begin{align}
f(s') &\geq \frac{1}{k}\cdot \mathrm{OPT}+\left(1-\frac{\sigma_{o,|s|}}{k}\right)\cdot f(s)\\
& \geq \frac{1}{k}\cdot \mathrm{OPT}+\left(1-\frac{\sigma_{o,k-1}}{k}\right)\cdot f(s)\\
& \geq \frac{1}{k}\cdot \mathrm{OPT}+\left(1-\frac{\sigma_{o,k-1}}{k}\right)\cdot \frac{1}{\sigma_{o,k-1}}\left(1-\left(1-\frac{\sigma_{o,k-1}}{k}\right)^i\right) \cdot \mathrm{OPT}\\
& = \frac{1}{\sigma_{o,k-1}}\left(1-\left(1-\frac{\sigma_{o,k-1}}{k}\right)^{i+1}\right)\cdot \mathrm{OPT},
\end{align}
where the second inequality holds because $\sigma_{o,|s|} \leq \sigma_{o,k-1}$, led by $\sigma_{o,m}$ increasing with $m$ in Remark~\ref{remark1} and $|s|\leq i<k$, and the third inequality holds by applying Eq.~(\refeq{eq:inductive}). Because $|s'|=|s|+1 \leq i+1$, $s'$ will be included into the population $P$; otherwise, $s'$ must be dominated by one sequence in $P$ (line~5 of Algorithm~\ref{alg:POSeqSel}), and this implies that $J_{\max}$ has already been larger than $i$, contradicting the assumption $J_{\max}=i$. After including $s'$ into $P$, $J_{\max} \geq i+1$. Let $P_{\max}$ denote the largest size of the population $P$ during the run of GSEMO. Thus, $J_{\max}$ can increase by at least 1 in one iteration with probability at least
\begin{align}\label{eq:mid13}
\frac{1}{P_{\max}} \cdot \frac{1}{e} \cdot \frac{1}{2} \cdot \frac{1}{n} \cdot \frac{1}{|s|+1} \geq \frac{1}{2e(i+1)nP_{\max}},\end{align}
where $1/P_{\max}$ is a lower bound on the probability of selecting $s$ in line~3 of Algorithm~\ref{alg:POSeqSel} due to uniform selection, $1/e$ is the probability of the number $r=1$ in Definition~\ref{def:mutate} (i.e., insertion or deletion is performed only once in mutation) according to the Poisson distribution with parameter $\lambda=1$, $1/2$ is the probability of performing insertion in mutation, and $(1/n)\cdot (1/(|s|+1))$ is the probability of selecting a specific item from $V$ and adding it to the end of $s$ in insertion (i.e., Definition~\ref{def:insert}). Then, it needs at most $2e(i+1)nP_{\max}$ iterations in expectation to increase $J_{\max}$. We thus get that the expected number of iterations until $J_{\max}=k$ is at most \begin{align}\sum\limits^{k-1}_{i=0} 2e(i+1)nP_{\max}=ek(k+1)nP_{\max}.\end{align}

By the procedure of GSEMO, we know that the sequences maintained in the population $P$ must be incomparable. Thus, each value of one objective can correspond to at most one sequence in $P$. Because the sequences with $|s| \geq 2k$ have $-\infty$ value on the first objective, they must be excluded from $P$. Thus, $P_{\max} \leq 2k$, implying that the expected number $\mathbb{E}[T]$ of iterations for finding a desired sequence is at most $2ek^2(k+1)n$. The theorem holds.
\end{proof}

Because the problem class of maximizing prefix$+$suffix monotone and prefix submodular functions in Definition~\ref{def:prob-presuf} is a subclass of that of maximizing prefix monotone submodular functions in Definition~\ref{def:prob-pre}, as shown in Figure~\ref{fig_relation-mon-sub}, we can directly apply Theorem~\ref{theo-prefix} to derive that the GSEMO with $\mathbb{E}[T] \leq 2ek^2(k+1)n$ achieves the approximation guarantee of $(1/\sigma_{o,k-1})(1-e^{-\sigma_{o,k-1}})$. According to the definition of curvature, i.e., Definition~\ref{def-strcur}, we have
\begin{align}
\sigma_{o,k-1}=\max_{t \in \mathcal{S}, 0<|t|\leq k-1} \left\{1-\frac{f(t \oplus o)-f(o)}{f(t)}\right\}.
\end{align}
As $f$ now satisfies the suffix monotonicity, $f(t \oplus o)-f(o) \geq 0$, and thus $\sigma_{o,k-1} \leq 1$. Because $(1/x)(1-e^{-x})$ is decreasing with $x$ when $x >0$, we have $(1/\sigma_{o,k-1})(1-e^{-\sigma_{o,k-1}}) \geq 1-1/e$. Thus, Theorem~\ref{theo-prefix-suffix} holds. Note that the obtained $(1-1/e)$-approximation guarantee reaches the best known one, which was previously obtained by the greedy algorithm~\citep{streeter2008online}.

\begin{theorem}\label{theo-prefix-suffix}
For the problem class of maximizing prefix$+$suffix monotone and prefix submodular functions in Definition~\ref{def:prob-presuf}, the GSEMO with $\mathbb{E}[T] \leq 2ek^2(k+1)n$ finds a sequence $s \in \mathcal{S}$ with $|s| \leq k$ and \begin{align}f(s) \geq  (1-1/e)\cdot \mathrm{OPT}.\end{align}
\end{theorem}

Because the problem class of maximizing subsequence monotone submodular functions in Definition~\ref{def:prob-subsequence} is a subclass of that of maximizing prefix$+$suffix monotone and prefix submodular functions in Definition~\ref{def:prob-presuf}, as shown in Figure~\ref{fig_relation-mon-sub}, we can directly apply Theorem~\ref{theo-prefix-suffix} to derive Theorem~\ref{theo-subsequence}. Note that the $(1-1/e)$-approximation guarantee has already reached the best known one, previously obtained by the greedy algorithm~\citep{alaei2021maximizing}.

\begin{theorem}\label{theo-subsequence}
For the problem class of maximizing subsequence monotone submodular functions in Definition~\ref{def:prob-subsequence}, the GSEMO with $\mathbb{E}[T] \leq 2ek^2(k+1)n$ finds a sequence $s \in \mathcal{S}$ with $|s| \leq k$ and \begin{align}f(s) \geq  (1-1/e)\cdot \mathrm{OPT}.\end{align}
\end{theorem}

\section{Theoretical Analysis on Maximizing Weakly Monotone and Strongly Submodular Functions}\label{sec-theory2}

Next, we prove in Theorem~\ref{theo-weak} that for the problem class of maximizing weakly monotone and strongly submodular functions in Definition~\ref{def:prob-weak}, the GSEMO with $\mathbb{E}[T] \leq 2ek^2(k+1)n$ can achieve a $(1-1/e)$-approximation guarantee, which reaches the best known one previously obtained by the generalized greedy algorithm~\citep{bernardini2020through}. The proof relies on Lemma~\ref{lem-prob1-onestep}, i.e., for any sequence $s$, there always exists one item whose insertion into a specific position of $s$ can bring an increment on $f$, proportional to the current distance to the optimum.

\begin{lemma}\label{lem-prob1-onestep}
Let $f:\mathcal{S} \rightarrow \mathbb{R}_+$ be a weakly monotone and strongly submodular function. For any sequence $s \in \mathcal{S}$, inserting a specific item $v \in V$ into a specific position of $s$ can generate a sequence $s'$ with
\begin{align}f(s')-f(s) \geq (\mathrm{OPT}-f(s))/k.\end{align}
\end{lemma}
\begin{proof}
Let $o$ be an optimal sequence of Eq.~(\refeq{eq:prob-weak}), i.e., $f(o)=\mathrm{OPT}$. Given any $s \in \mathcal{S}$, let $w^*$ denote a sequence such that
\begin{align}\label{eq:mid11}
w^* \in \arg\max\nolimits_{w \in \mathcal{S}} f(w) \quad s.t. \quad s \sqsubseteq_{\mathrm{subseq}} w, o \sqsubseteq_{\mathrm{subseq}} w \; \text{and}\; |w| \leq |s|+|o|.
\end{align}
Because $s \sqsubseteq_{\mathrm{subseq}} w^*=(w^*_1,w^*_2,\ldots,w^*_{|w^*|})$, we can divide the indices $\{1,2,\ldots,|w^*|\}$ into two parts: $\{i_1,i_2,\ldots,i_{|s|}\}$ and $\{j_1,j_2,\ldots,j_{|w^*|-|s|}\}$, where $i_1<i_2<\cdots<i_{|s|}$ are the indices satisfying that $\forall l \in \{1,2,\ldots,|s|\}: w^*_{i_l}=s_l$, i.e., the subsequence $(w^*_{i_1},w^*_{i_2},\ldots,w^*_{i_{|s|}})$ of $w^*$ is equal to $s$, and $j_1<j_2<\cdots<j_{|w^*|-|s|}$ are the remaining indices. For any $l \in \{0,1,\ldots,|w^*|-|s|\}$, let $w^{(l)}$ denote the subsequence of $w^*$ by deleting the items $w^*_{j_{|w^*|-|s|}},w^*_{j_{|w^*|-|s|-1}},\ldots,w^*_{j_{l+1}}$. It is easy to see that $w^{(|w^*|-|s|)}=w^{*}$ and $w^{(0)}=s$. Thus, we have
\begin{align}\label{eq:mid8}
f(w^*)-f(s)&=\sum\limits^{|w^*|-|s|-1}_{l=0} f(w^{(l+1)})-f(w^{(l)}).
\end{align}
Compared with $w^{(l)}$, $w^{(l+1)}$ has one more item, i.e., $w^*_{j_{l+1}}$. For $w^{(l+1)}$, we use $w^{(l+1)}_{\mathrm{prefix}}$ and $w^{(l+1)}_{\mathrm{suffix}}$ to denote the prefix before $w^*_{j_{l+1}}$ and the suffix after $w^*_{j_{l+1}}$, respectively. That is, $w^{(l+1)}= w^{(l+1)}_{\mathrm{prefix}} \oplus w^*_{j_{l+1}} \oplus w^{(l+1)}_{\mathrm{suffix}}$. For $w^{(l)}$, it holds that $w^{(l)}=w^{(l+1)}_{\mathrm{prefix}} \oplus w^{(l+1)}_{\mathrm{suffix}}$. We generate two sequences by deleting the common items $w^*_{j_{1}},w^*_{j_{2}},\ldots,w^*_{j_{l}}$, which must appear in $w^{(l+1)}_{\mathrm{prefix}}$, from $w^{(l)}$ and $w^{(l+1)}$. The first generated sequence is just $s$, and the second one is a combination of $s$ and $w^*_{j_{l+1}}$, denoted as $s^{(l+1)}$. For $s^{(l+1)}$, we similarly use $s^{(l+1)}_{\mathrm{prefix}}$ and $s^{(l+1)}_{\mathrm{suffix}}$ to denote the prefix before $w^*_{j_{l+1}}$ and the suffix after $w^*_{j_{l+1}}$, respectively. That is, $s^{(l+1)}= s^{(l+1)}_{\mathrm{prefix}} \oplus w^*_{j_{l+1}} \oplus s^{(l+1)}_{\mathrm{suffix}}$. We also have $s^{(l+1)}_{\mathrm{prefix}} \oplus s^{(l+1)}_{\mathrm{suffix}}=s$. According to the way of generating $s^{(l+1)}$, it holds that $s^{(l+1)}_{\mathrm{prefix}} \sqsubseteq_{\mathrm{subseq}} w^{(l+1)}_{\mathrm{prefix}}$ and $s^{(l+1)}_{\mathrm{suffix}} = w^{(l+1)}_{\mathrm{suffix}}$. Thus, by the strong submodularity of $f$ (i.e., Definition~\ref{def-strong-sub}), we have $\forall l \in \{0,1,\ldots,|w^*|-|s|-1\}$,
\begin{align}\label{eq:mid9}
f(w^{(l+1)})-f(w^{(l)})&= f(w^{(l+1)}_{\mathrm{prefix}} \oplus w^*_{j_{l+1}} \oplus w^{(l+1)}_{\mathrm{suffix}})-f(w^{(l+1)}_{\mathrm{prefix}} \oplus w^{(l+1)}_{\mathrm{suffix}})
\\
&=f(w^{(l+1)}_{\mathrm{prefix}} \oplus w^*_{j_{l+1}} \oplus s^{(l+1)}_{\mathrm{suffix}})-f(w^{(l+1)}_{\mathrm{prefix}} \oplus s^{(l+1)}_{\mathrm{suffix}})\\
&\leq f(s^{(l+1)}_{\mathrm{prefix}} \oplus w^*_{j_{l+1}} \oplus s^{(l+1)}_{\mathrm{suffix}})-f(s^{(l+1)}_{\mathrm{prefix}} \oplus s^{(l+1)}_{\mathrm{suffix}})\\
&=f(s^{(l+1)})-f(s).
\end{align}
Applying Eq.~(\refeq{eq:mid9}) to Eq.~(\refeq{eq:mid8}) leads to
\begin{align}\label{eq:mid10}
f(w^*)-f(s)\leq \sum\limits^{|w^*|-|s|-1}_{l=0} (f(s^{(l+1)})-f(s)).
\end{align}
Let $s'=\arg\max_{s': s \sqsubseteq_{\mathrm{subseq}} s', |s'|=|s|+1} f(s')$. Because $s \sqsubseteq_{\mathrm{subseq}} s^{(l+1)}$ and $|s^{(l+1)}|=|s|+1$, we have $f(s^{(l+1)}) \leq f(s')$, implying
\begin{align}
f(w^*)-f(s)\leq (|w^*|-|s|) \cdot (f(s')-f(s)).
\end{align}
By the definition of $w^*$, i.e., Eq.~(\refeq{eq:mid11}), we have $|w^*| \leq |s|+|o| \leq |s|+k$. Furthermore, by Eq.~(\refeq{eq:mid11}) and the weak monotonicity of $f$ (i.e., Definition~\ref{def-weak-mon}), we have $f(w^*)\geq f(s)$. Thus,
\begin{align}
f(s')-f(s) \geq \frac{1}{k}\cdot (f(w^*)-f(s)).
\end{align}
Again, by Eq.~(\refeq{eq:mid11}) and the weak monotonicity of $f$, we have $f(w^*)\geq f(o)=\mathrm{OPT}$. Thus,
\begin{align}
f(s')-f(s) \geq \frac{1}{k}\cdot (\mathrm{OPT}-f(s)),
\end{align}
implying that the lemma holds.
\end{proof}

The proof of Theorem~\ref{theo-weak} is similar to that of Theorem~\ref{theo-prefix}. The main difference is that a different inductive inequality on $f$ is used in the definition of the quantity $J_{\max}$. For concise illustration, we will mainly show the difference in the proof of Theorem~\ref{theo-weak}.

\begin{theorem}\label{theo-weak}
For the problem class of maximizing weakly monotone and strongly submodular functions in Definition~\ref{def:prob-weak}, the GSEMO with $\mathbb{E}[T] \leq 2ek^2(k+1)n$ finds a sequence $s \in \mathcal{S}$ with $|s| \leq k$ and \begin{align}f(s) \geq  (1-1/e)\cdot \mathrm{OPT}.\end{align}
\end{theorem}
\begin{proof}
The proof is similar to that of Theorem~\ref{theo-prefix}. We use a different $J_{\max}$, which is defined as
\begin{align}
J_{\max}=\max\{j \in \{0,1,\ldots,k\} \mid \exists s \in P,  |s| \leq j \wedge f(s) \geq (1-(1-1/k)^j) \cdot \mathrm{OPT}\}.
\end{align}
It is easy to verify that $J_{\max}=k$ implies that the desired approximation guarantee is reached, since there must exist one sequence $s$ in the population $P$ satisfying that $|s| \leq k$ and $f(s) \geq (1-(1-\frac{1}{k})^k) \cdot \mathrm{OPT} \geq (1-1/e) \cdot \mathrm{OPT}$. Assume that currently $J_{\max}=i<k$ and $s$ is a corresponding sequence, i.e., $|s| \leq i$ and
\begin{align}\label{eq:mid4}
f(s) \geq (1-(1-1/k)^i) \cdot \mathrm{OPT}.
\end{align}
By Lemma~\ref{lem-prob1-onestep}, we know that there exists one specific item, the insertion of which into a specific position of $s$ can generate a new sequence $s'$, satisfying \begin{align}\label{eq:mid12}f(s')-f(s) \geq (\mathrm{OPT}-f(s))/k.\end{align} By applying Eq.~(\refeq{eq:mid4}) to Eq.~(\refeq{eq:mid12}), we have
\begin{align}
f(s') \geq (1-(1-1/k)^{i+1})\cdot \mathrm{OPT}.
\end{align}
Note that for the insertion operator in Definition~\ref{def:insert}, the probability of inserting a specific item into a specific position of $s$ is the same as that of appending a specific item to the end of $s$, which is $(1/n)\cdot (1/(|s|+1))$. Thus, by using the same analysis as Eq.~(\refeq{eq:mid13}), we get that the probability of increasing $J_{\max}$ in one iteration of the GSEMO is at least $1/(2e(i+1)nP_{\max}) \geq 1/(4e(i+1)kn)$, where the inequality holds by $P_{\max}\leq 2k$. This implies that the expected number of iterations until $J_{\max}=k$ is at most $\sum^{k-1}_{i=0}4e(i+1)kn=2ek^2(k+1)n$. Thus, the theorem holds.
\end{proof}

\section{Theoretical Analysis on Maximizing DAG Monotone Submodular Functions}\label{sec-theory3}

Finally, we examine the problem class of maximizing DAG monotone submodular functions in Definition~\ref{def:prob-dag}. For a DAG monotone submodular objective function $f$, we know from Definition~\ref{def:dagms} that there exists a DAG $G=(V,E)$ (not counting self-cycles) and a monotone submodular set function $h: 2^{E} \rightarrow \mathbb{R}_+$ such that for any $s \in \mathcal{S}$, $f(s)=h(E(s))$, where $E(s)=\{(s_i,s_j) \mid (s_i,s_j)\in E, i \leq j\}$ is the set of edges induced by $s$ on $G$. In this case, the prefix submodular property is not necessarily satisfied by $f$, and the greedy algorithm fails to achieve a constant approximation guarantee. Tschiatschek et al.~\cite{tschiatschek2017selecting} then developed the \textsc{OMegA} algorithm by exploiting the DAG property of the graph $G$, which obtains the approximation guarantee of $1-e^{-1/(2\Delta)}$, where $\Delta=\min\{\Delta_{\text{in}}, \Delta_{\text{out}}\} \geq 1$, and $\Delta_{\text{in}}$, $\Delta_{\text{out}}$ are the largest indegree and outdegree of the items in $G$, respectively.

Let $V(s)$ denote the set of items present in a sequence $s \in \mathcal{S}$. In the implementation of GSEMO, we also utilize the DAG property of the graph $G$: when computing $f(s)$, we directly use the $f$ value of the optimal ordering for $V(s)$, i.e., $f(s)=f(\textsc{Reorder}(V(s)))=h(E(\textsc{Reorder}(V(s))))$; when the algorithm terminates, we output $\textsc{Reorder}(V(s))$ instead of $s$. Note that $\textsc{Reorder}(\cdot)$ denotes the optimal sequence for a set of items, which can be decided by first computing a topological ordering of the graph $G$ and then sorting the set of items according to that order. We prove in Theorem~\ref{theo-dag} that the GSEMO with $\mathbb{E}[T] \leq 4ek^2n^2$ can achieve an approximation guarantee of $1-e^{-1/2}$, which is better than the best known one, i.e., $1-e^{-1/(2\Delta)}$~\citep{tschiatschek2017selecting}, since $\Delta \geq 1$ and is usually much larger than 1. The proof relies on Lemma~\ref{lem-prob3-onestep}, that for any $s \in \mathcal{S}$, there always exist one or two items, the insertion of which into $s$ can bring an improvement on $f$, proportional to the current distance to the optimum.

\begin{lemma}\label{lem-prob3-onestep}
Let $f:\mathcal{S} \rightarrow \mathbb{R}_+$ be a DAG monotone submodular function. For any sequence $s \in \mathcal{S}$, there exists one item $v \in V \setminus V(s)$ or two items $u,v \in V \setminus V(s)$ such that inserting $v$ or $u,v$ into any positions of $s$ generates a sequence $s'$ with
\begin{align}\label{eq:lemma3}
f(s')-f(s) \geq (\mathrm{OPT}-f(s))/k.
\end{align}
\end{lemma}
\begin{proof}
The proof relies on an auxiliary set function $g: 2^E \rightarrow \mathbb{R}$. Let $V(X)$ denote the set of items covered by an edge set $X \subseteq E$. We define $g$ as for any $X \subseteq E$, $g(X)=h(E(\textsc{Reorder}(V(X))))$. Note that $g$ is monotone and submodular, as proved in Lemma~1 of~\citep{tschiatschek2017selecting}.

Let $o$ be an optimal sequence of Eq.~(\refeq{eq:prob-dag}), and let $X^*\in\arg\min_{X \subseteq E, V(X)=V(o)} |X|$, i.e., $X^*$ is the smallest edge set which covers the item set $V(o)$. It is easy to see that $|X^*| \leq k$, since $|o| \leq k$ and one edge can cover at least one item. For any $s \in \mathcal{S}$, we have
\begin{align}
g(X^*)-g(E(s)) &\leq g(E(s) \cup X^*)-g(E(s)) \\
&\leq \sum\limits_{e \in X^* \setminus E(s)} \big(g(E(s) \cup \{e\}) -g(E(s))\big),
\end{align}
where the first inequality holds by the monotonicity of $g$, i.e, $\forall X \subseteq Y: g(X) \leq g(Y)$, and the second inequality holds by the submodularity of $g$, i.e, $\forall X \subseteq Y: g(Y)-g(X) \leq \sum\nolimits_{e \in Y \setminus X} g(X \cup \{e\})-g(X)$. Let $e^*\in\arg\max_{e \in X^* \setminus E(s)} g(E(s) \cup \{e\})$. Since $|X^* \setminus E(s)| \leq |X^*| \leq k$, we have
\begin{align}\label{eq:mid15}g(E(s) \cup \{e^*\})-g(E(s)) \geq (g(X^*)-g(E(s)))/k.\end{align}
By the definitions of $f$ and $g$, we easily verify that for any $X \subseteq E$ and $s \in \mathcal{S}$, if $V(X)=V(s)$, $g(X)=h(E(\textsc{Reorder}(V(X))))=h(E(\textsc{Reorder}(V(s))))=f(s)$. Thus,
\begin{align}\label{eq:mid14}g(E(s))=f(s) \;\wedge\; g(X^*)=f(o)=\mathrm{OPT},\end{align}
since $V(E(s))=V(s)$ and $V(X^*)=V(o)$. Let $s'$ be any sequence with $V(s')=V(E(s) \cup \{e^*\})$. Then, we have \begin{align}\label{eq:mid16}f(s')=g(E(s) \cup e^*).\end{align}
Applying Eqs.~(\refeq{eq:mid14}) and~(\refeq{eq:mid16}) to Eq.~(\refeq{eq:mid15}) leads to
\begin{align}
f(s')-f(s) \geq (\mathrm{OPT}-f(s))/k.
\end{align}
Because $V(s')=V(E(s) \cup \{e^*\})=V(s) \cup V(\{e^*\})$ and $e^*$ introduces one or two new items, the lemma holds.
\end{proof}

\begin{theorem}\label{theo-dag}
For the problem class of maximizing DAG monotone submodular functions in Definition~\ref{def:prob-dag}, the GSEMO with $\mathbb{E}[T] \leq 4ek^2n^2$ finds a sequence $s \in \mathcal{S}$ with $|s| \leq k$ and \begin{align}f(s) \geq  (1-e^{-1/2})\cdot \mathrm{OPT}.\end{align}
\end{theorem}
\begin{proof}
The proof is similar to that of Theorem~\ref{theo-prefix}. We use a different $J_{\max}$, which is defined as
\begin{align}
&J_{\max}=\max\{j \in \{0,1,\ldots,k\} \mid \exists s \in P, |s| \leq j \wedge f(s) \geq (1-(1-1/k)^{\lceil j/2 \rceil}) \cdot \mathrm{OPT}\}.
\end{align}
We analyze the expected number of iterations until $J_{\max} \geq k-1$, which implies that there exists one sequence $s$ in $P$ satisfying that $|s| \leq k$ and \begin{align}f(s) \geq \left(1-\left(1-\frac{1}{k}\right)^{\lceil (k-1)/2 \rceil}\right) \cdot \mathrm{OPT} \geq \left(1-e^{-\frac{k-1}{2k}}\right) \cdot \mathrm{OPT},\end{align} which is nearly $(1-e^{-1/2}) \cdot \mathrm{OPT}$ for sufficiently large $k$.

As the analysis in the proof of Theorem~\ref{theo-prefix}, $J_{\max}$ is initially $0$ and never decreases. Assume that currently $J_{\max}=i<k-1$ and $s$ is a corresponding sequence, i.e., $|s| \leq i$ and
\begin{align}\label{eq:mid5}
&f(s)\geq (1-(1-1/k)^{\lceil i/2 \rceil}) \cdot \mathrm{OPT}.
\end{align}
By Lemma~\ref{lem-prob3-onestep}, we know that there exist one or two items, the insertion of which into $s$ can generate a new sequence $s'$, which satisfies that \begin{align}\label{eq:mid17}f(s')-f(s) \geq \frac{1}{k} \cdot (\mathrm{OPT}-f(s)).\end{align} By applying Eq.~(\refeq{eq:mid5}) to Eq.~(\refeq{eq:mid17}), we get
\begin{align}
f(s') \geq (1-(1-1/k)^{\lceil (i+2)/2 \rceil}) \cdot \mathrm{OPT}.
\end{align}
Because $|s'|\leq |s|+2 \leq i+2$, $s'$ will be included into the population $P$, making $J_{\max} \geq i+2$. The probability of inserting one specific item into $s$ is obviously larger than that of inserting two specific items into $s$, which is at least
\begin{align}\label{eq:mid18}\frac{1}{P_{\max}} \cdot \frac{1}{2e} \cdot \frac{1}{2^2} \cdot \frac{2}{(n-|s|)(n-|s|-1)} \geq \frac{1}{4en^2P_{\max}},\end{align} where $1/P_{\max}$ is a lower bound on the probability of selecting $s$ in line~3 of Algorithm~\ref{alg:POSeqSel}, $1/(2e)$ is the probability of the number $r=2$ in Definition~\ref{def:mutate} (i.e, insertion or deletion is implemented twice in mutation), $1/2^2$ is the probability of performing insertion twice in mutation, and $2/((n-|s|)(n-|s|-1))=1/\binom{n-|s|}{2}$ is the probability of selecting the two specific items for insertion. Note that the items in a sequence cannot be repeated here, and thus the insertion operator in Definition~\ref{def:insert} selects an item from $V/V(s)$. Combining $P_{\max} \leq 2k$ with Eq.~(\refeq{eq:mid18}), $J_{\max}$ can increase by at least $2$ in one iteration with probability at least $1/(8ekn^2)$. Thus, the expected number of iterations until $J_{\max}\geq k-1$ is at most \begin{align}8ekn^2 \cdot \lceil (k-1)/2\rceil \leq 4ek^2n^2,\end{align} implying that the theorem holds.
\end{proof}

Note that it was proved in~\citep{tschiatschek2017selecting} that the \textsc{OMegA} algorithm (i.e., Algorithm~\ref{alg:OMegA}) achieves an approximation guarantee of $1-e^{-1/(2\Delta)}$. However, by using Lemma~\ref{lem-prob3-onestep}, we can actually prove that the \textsc{OMegA} algorithm also achieves the $(1-e^{-1/2})$-approximation guarantee. According to the procedure of Algorithm~\ref{alg:OMegA}, the \textsc{OMegA} algorithm iteratively selects an edge with the largest marginal improvement, such that the number of covered items is not larger than $k$. As an edge can cover at most two new items, the inductive inequality Eq.~(\refeq{eq:lemma3}) in Lemma~\ref{lem-prob3-onestep} can be applied at least $\lfloor k/2 \rfloor$ times, leading to an approximation guarantee of $1-(1-1/k)^{\lfloor k/2 \rfloor}=1-(1-1/k)^{\lceil (k-1)/2 \rceil}\geq 1-e^{-\frac{k-1}{2k}}$, which is nearly $1-e^{-1/2}$ for sufficiently large $k$.

Next we briefly explain why the analysis in~\citep{tschiatschek2017selecting} leads to only an approximation guarantee of $1-e^{-1/(2\Delta)}$. In the proof of Lemma~\ref{lem-prob3-onestep}, we use the edge set $X^*\in\arg\min_{X \subseteq E, V(X)=V(o)} |X|$, satisfying that $g(X^*)=f(o)=\mathrm{OPT}$ and $|X^*| \leq k$, where $o$ is an optimal sequence. But in the analysis of~\citep{tschiatschek2017selecting} (i.e., their proof of Theorem~2), the edge set $S^{\mathrm{opt}} \in \arg\max_{X \subseteq E, |X|\leq \Delta \cdot k}g(X)$ is used, where $\Delta=\min\{\Delta_{\text{in}}, \Delta_{\text{out}}\}$, and $\Delta_{\text{in}}$, $\Delta_{\text{out}}$ are the largest indegree and outdegree of the items in the graph $G$, respectively. As a subgraph of $G$ with at most $k$ items contains at most $\Delta \cdot k$ edges, it holds that $g(S^{\mathrm{opt}}) \geq f(o)=\mathrm{OPT}$. Furthermore, $|S^{\mathrm{opt}}|$ is upper bounded by $\Delta \cdot k$. Based on these two points, they derived the inductive inequality
\begin{align}\label{eq:lemma3-v}
f(s')-f(s) \geq (\mathrm{OPT}-f(s))/(\Delta \cdot k),
\end{align}
which can also be verified by following our proof of Lemma~\ref{lem-prob3-onestep} and replacing $X^*$ with $S^{\mathrm{opt}}$. By applying Eq.~(\refeq{eq:lemma3-v}) $\lfloor k/2 \rfloor$ times, they derived an approximation guarantee of $1-(1-1/(\Delta \cdot k))^{\lfloor k/2 \rfloor}\geq 1-e^{-\frac{k-1}{2\Delta \cdot k}}$, which is nearly $1-e^{-1/(2\Delta)}$ for sufficiently large $k$. Thus, our analysis for the GSEMO leads to a tighter inductive inequality (i.e., Eq.~(\refeq{eq:lemma3})) than Eq.~(\refeq{eq:lemma3-v}), which can be used to improve the analysis of the existing \textsc{OMegA} algorithm~\citep{tschiatschek2017selecting}.

\section{Experiments}\label{sec-exp}

In the above three sections, we have theoretically shown that the GSEMO can always reach or improve the best-known polynomial-time approximation guarantee for the previously studied problem classes of maximizing monotone submodular functions over sequences, including maximizing prefix monotone submodular functions, maximizing weakly monotone and strongly submodular functions, and maximizing DAG monotone submodular functions. In this section, we will investigate the empirical performance of the GSEMO by experiments on two applications of each of these problem classes. Note that the upper bound on the expectation of the number $T$ of iterations (i.e., $\mathbb{E}[T]$) of the GSEMO for achieving a good approximation has been derived in theoretical analysis. To guarantee the good approximation with high probability, $T$ may need to be set as at least multiple times of the derived upper bound on $\mathbb{E}[T]$, but we only use this upper bound as the budget, to make the running time acceptable. In fact, our experiments will show that such a budget has been sufficient to guarantee a good approximation of the GSEMO in practice. In the bi-objective problem reformulation Eq.~(\refeq{def-CO-BO}), $f_1(s)$ has been set to $-\infty$ for $|s|\geq 2k$, with the goal of allowing infeasible sequences (i.e., sequences with $k<|s|<2k$) with small constraint violation degree to participate in the evolutionary process. In the experiments, we will also implement the bi-objective problem reformulation where $f_1(s)=-\infty$ for $|s|> k$ (i.e., infeasible sequences are always discarded), and the algorithm is denoted as GSEMO$_k$ accordingly. The comparison between the GSEMO and GSEMO$_k$ will be discussed in Section~\ref{subsec-discussion}.

\subsection{Maximizing Prefix Monotone Submodular Functions}

For the problem class of maximizing prefix monotone submodular functions in Definition~\ref{def:prob-pre}, we consider the two applications of accomplishing tasks and maximizing information gain, introduced in~\citep{zhang2016string}. We compare the GSEMO with the previous best algorithm, i.e., the greedy algorithm in Algorithm~\ref{alg:greedy}. The number $T$ of iterations of the GSEMO is set to $2ek^2(k+1)n$ as suggested by Theorem~\ref{theo-prefix}. For each application, we compare their performance by varying the budget $k$ or the number $n$ of items. For the former, $n$ is fixed to 500 and $k$ is set as $\{10,12,\ldots,30\}$. For the latter, $k$ is fixed to 20 and $n$ is set as $\{100,200,\ldots,1000\}$. For each combination of $n$ and $k$, we randomly generate 50 problem instances, and report the number of instances where the GSEMO wins, ties or loses against the greedy algorithm, followed by the sign-test~\citep{demsar:06} with confidence level $0.05$.

\subsubsection{Accomplishing Tasks}\label{subsec-exp-1-1}

The application of selecting a sequence of actions to maximize the expected fraction of accomplished tasks~\citep{zhang2016string} is first considered. Given $m$ tasks, $n$ actions and a sequence $s=(s_1,s_2,\ldots,s_l)$ of actions, the objective function is \begin{align}f(s)=\frac{1}{m}\sum\limits^m_{i=1} \left(1-\prod\limits^l_{j=1} \left(1-p^j_i(s_j)\right)\right),\end{align} where $p^j_i(s_j)$ is the probability of accomplishing task $i$ by performing action $s_j$ at stage $j$. Here one item corresponds to one action.

We set $m=50$, and each probability $p^j_i(s_j)$ is randomly sampled from $[0,0.2]$. The results are shown in Tables~\ref{tab-prefix-task-1} and~\ref{tab-prefix-task-2}. We can observe that for each setting of $n$ and $k$, the GSEMO achieves a better average objective value than the greedy algorithm, and is further significantly better by the sign-test with confidence level $0.05$. The results are also consistent with that the optimal objective value increases with $k$ and $n$. Note that we do not report the standard deviation of the objective value, which is not fully meaningful because the objective values are obtained by running the GSEMO on 50 different problem instances.

\begin{table*}[t!]\caption{Comparison between the GSEMO, the GSEMO$_k$ and the greedy algorithm for the application of accomplishing tasks with $n=500$ and $k \in \{10,12,\ldots,30\}$. For each $k$, the average objective value over 50 random problem instances is reported, and the largest one is bolded. The count of win/tie/loss denotes the number of problem instances where the GSEMO has a larger/equal/smaller objective value than the greedy algorithm, and significant cells by the sign-test with confidence level $0.05$ are bolded.}\label{tab-prefix-task-1}
\footnotesize
\begin{center}
\setlength{\tabcolsep}{0.75mm}{
\begin{tabular}{c|ccccccccccc}
\hline
$k$ & $10$ &$12$&$14$ & $16$ & $18$ & $20$ &$22$ &$24$&$26$ & $28$ & $30$\\
\hline
GSEMO & \bf{0.7364} & \bf{0.7981} & \bf{0.8454} & \bf{0.8818} & \bf{0.9095} & \bf{0.9309} & \bf{0.9472} & \bf{0.9596} & \bf{0.9691} & \bf{0.9763}  & \bf{0.9819} \\
GSEMO$_k$ & 0.7363 & \bf{0.7981} & \bf{0.8454} & 0.8817 & \bf{0.9095} & \bf{0.9309} & \bf{0.9472} & \bf{0.9596} & \bf{0.9691} & \bf{0.9763} & \bf{0.9819} \\
Greedy & 0.7362 & 0.7979 & 0.8452 & 0.8815& 0.9093& 0.9307 & 0.9470 & 0.9595 & 0.9690& 0.9762 & 0.9818\\
\hline
win/tie/loss & \bf{26/14/10} & \bf{33/7/10} & \bf{36/5/9} & \bf{39/6/5} & \bf{39/2/9} & \bf{39/4/7} & \bf{41/1/8} & \bf{42/1/7} & \bf{42/2/6} & \bf{44/1/5} & \bf{45/0/5} \\
\hline
\end{tabular}}
\end{center}
\end{table*}

\begin{table*}[t!]\caption{Comparison between the GSEMO, the GSEMO$_k$ and the greedy algorithm for the application of accomplishing tasks with $n\in\{100,200,\ldots,1000\}$ and $k=20$. For each $n$, the average objective value over 50 random problem instances is reported, and the largest one is bolded. The count of win/tie/loss denotes the number of problem instances where the GSEMO has a larger/equal/smaller objective value than the greedy algorithm, and significant cells by the sign-test with confidence level $0.05$ are bolded.}\label{tab-prefix-task-2}
\footnotesize
\begin{center}
\setlength{\tabcolsep}{0.75mm}{
\begin{tabular}{c|cccccccccc}
\hline
$n$ & $100$ & $200$ & $300$ & $400$ & $500$ & $600$ & $700$ & $800$ & $900$ & $1000$\\
\hline
GSEMO & \bf{0.9238} & \bf{0.9269} & \bf{0.9287} & \bf{0.9300} & \bf{0.9309}  & \bf{0.9316} & \bf{0.9322} & \bf{0.9326} & \bf{0.9329} & \bf{0.9333}\\
GSEMO$_k$ & \bf{0.9238} & \bf{0.9269} & 0.9286 & \bf{0.9300} & \bf{0.9309} & 0.9315 & \bf{0.9322} & \bf{0.9326} & \bf{0.9329} & \bf{0.9333}\\
Greedy & 0.9236 & 0.9267 & 0.9285 & 0.9298 & 0.9307 & 0.9313 & 0.9320 & 0.9324 & 0.9328 & 0.9331\\
\hline
win/tie/loss & \bf{40/2/8} & \bf{39/4/7} & \bf{45/1/4} & \bf{42/1/7} & \bf{39/4/7} & \bf{43/0/7} & \bf{43/0/7} & \bf{45/1/4} & \bf{41/1/8} & \bf{42/1/7}\\
\hline
\end{tabular}}
\end{center}
\end{table*}

\subsubsection{Maximizing Information Gain}\label{subsec-exp-1-2}

Next, we consider the task of maximizing information gain~\citep{zhang2016string}, widely used in Bayesian parameter estimation. Consider a parameter vector $\bm{\theta} \in \mathbb{R}^d$, which has a Gaussian prior distribution $\bm{\theta} \sim \mathcal{N}(\bm{\mu},\mathbf{P}_0)$. At each stage $i$, one can choose a measurement matrix $\mathbf{A}_i$ to obtain a noisy observation $\bm{y}_i$, where $\bm{y}_i=\mathbf{A}_i\bm{\theta}+\bm{\omega}_i$, the Gaussian noise $\bm{\omega}_i \sim \mathcal{N}(\bm{0},\sigma^2_i \mathbf{I})$, and $\mathbf{I}$ denotes the identity matrix. After choosing $l$ measurement matrices $\mathbf{A}_1,\mathbf{A}_2,\ldots,\mathbf{A}_l$, the covariance $\mathbf{P}_l$ of the posterior distribution of $\bm{\theta}$ satisfies
\begin{align}
\mathbf{P}^{-1}_l=\mathbf{P}^{-1}_0+\sum^{l}_{i=1}\mathbf{A}^{\mathrm{T}}_i (\sigma^2_i \mathbf{I})^{-1} \mathbf{A}_i,
\end{align}
and the entropy of the posterior distribution of $\bm{\theta}$ is
\begin{align}H_l=\frac{1}{2}\cdot \log \mathrm{det}(\mathbf{P}_l)+\log (2\pi e),\end{align} where $\mathrm{det}(\cdot)$ denotes the determinant of a matrix. The goal is to choose a sequence of measurement matrices minimizing the entropy of the posterior distribution of $\bm{\theta}$, or equivalently maximizing the information gain, i.e.,
\begin{align}\label{eq:mid20}
f((\mathbf{A}_1,\mathbf{A}_2,\ldots,\mathbf{A}_l))=H_0-H_l=\frac{1}{2}\cdot (\log \mathrm{det}(\mathbf{P}_0)-\log \mathrm{det}(\mathbf{P}_l)).
\end{align}
Here one item corresponds to one measurement matrix.

We set $d=2$, and select a measurement matrix from the set of diagonal positive-semidefinite $2 \times 2$ matrices with unit Frobenius norm: $\{\mathrm{Diag}(\sqrt{a},\sqrt{1-a})\mid a \in [0,1]\}$. In particular, we generate $1000$ measurement matrices by setting $a$ as $\{\frac{1}{1000},\frac{2}{1000},\ldots,1\}$, and order them as the absolute value of $\sqrt{a}-\sqrt{1-a}$ decreases. For each $n\in \{100,200,\ldots,1000\}$, we use the first $n$ measurement matrices in this order. For each setting of $n$ and $k$, 50 problem instances are generated by setting $\mathbf{P}_0$ and $\sigma_i$ randomly. $\mathbf{P}_0$ is set to a diagonal matrix with each entry on the diagonal randomly drawn from $[0,1)$. $\sigma_i$ is set to a random value in $[i-1,i)$. Note that the monotone non-decreasing property of $\sigma_i$ w.r.t. $i$ is to ensure the prefix submodularity of the objective $f$ in Eq.~(\refeq{eq:mid20}).

The results are shown in Tables~\ref{tab-prefix-gain-1} and~\ref{tab-prefix-gain-2}. We can observe that the average objective values obtained by the GSEMO and the greedy algorithm are very close, and even cannot be distinguished when they are rounded to four decimal points. But the number of win/tie/loss, which is based on the comparison between accurate objective values, still shows some advantages of the GSEMO. For each setting of $n$ and $k$, the GSEMO is never worse on all 50 problem instances, and can have several wins. By the sign-test with confidence level 0.05, the GSEMO can still be significantly better in 5 cases of Table~\ref{tab-prefix-gain-1} and 4 cases of Table~\ref{tab-prefix-gain-2}.

\begin{table*}[ht!]\caption{Comparison between the GSEMO, the GSEMO$_k$ and the greedy algorithm for the application of maximizing information gain with $n=500$ and $k \in \{10,12,\ldots,30\}$. For each $k$, the average objective value over 50 random problem instances is reported, and the largest one is bolded. The count of win/tie/loss denotes the number of problem instances where the GSEMO has a larger/equal/smaller objective value than the greedy algorithm, and significant cells by the sign-test with confidence level $0.05$ are bolded.}\label{tab-prefix-gain-1}
\footnotesize
\begin{center}
\setlength{\tabcolsep}{0.75mm}{
\begin{tabular}{c|ccccccccccc}
\hline
$k$ & $10$ &$12$&$14$ & $16$ & $18$ & $20$ &$22$ &$24$&$26$ & $28$ & $30$\\
\hline
GSEMO&\bf{1.8164}&\bf{1.8185}&\bf{1.8200}&\bf{1.8212}&\bf{1.8221}&\bf{1.8228}&\bf{1.8233}&\bf{1.8238}&\bf{1.8242} &\bf{1.8246} &\bf{1.8249}\\
GSEMO$_k$&\bf{1.8164} &\bf{1.8185} &\bf{1.8200} &\bf{1.8212} &\bf{1.8221} &\bf{1.8228} &\bf{1.8233} &\bf{1.8238} &\bf{1.8242} & \bf{1.8246} &\bf{1.8249}\\
Greedy&\bf{1.8164}&\bf{1.8185}&\bf{1.8200}&\bf{1.8212}&\bf{1.8221}&\bf{1.8228}&\bf{1.8233}&\bf{1.8238}&\bf{1.8242} &\bf{1.8246} &\bf{1.8249}\\
\hline
win/tie/loss&\bf{8/42/0}&5/45/0&5/45/0&\bf{7/43/0}&\bf{7/43/0}&\bf{7/43/0}&5/45/0&5/45/0&5/45/0&4/46/0&\bf{6/44/0}\\
\hline
\end{tabular}}
\end{center}
\end{table*}

\begin{table*}[ht!]\caption{Comparison between the GSEMO, the GSEMO$_k$ and the greedy algorithm for the application of maximizing information gain with $n\in\{100,200,\ldots,1000\}$ and $k=20$. For each $n$, the average objective value over 50 random problem instances is reported, and the largest one is bolded. The count of win/tie/loss denotes the number of problem instances where the GSEMO has a larger/equal/smaller objective value than the greedy algorithm, and significant cells by the sign-test with confidence level $0.05$ are bolded.}\label{tab-prefix-gain-2}
\footnotesize
\begin{center}
\setlength{\tabcolsep}{0.75mm}{
\begin{tabular}{c|cccccccccc}
\hline
$n$ & $100$ & $200$ & $300$ & $400$ & $500$ & $600$ & $700$ & $800$ & $900$ & $1000$\\
\hline
GSEMO&\bf{1.6814}&\bf{1.7387}&\bf{1.7766}&\bf{1.8034}&\bf{1.8228}&\bf{1.8365}&\bf{1.8457}&\bf{1.8512}&\bf{1.8537} &\bf{1.8542}\\
GSEMO$_k$&\bf{1.6814} &\bf{1.7387} &\bf{1.7766} &\bf{1.8034} &\bf{1.8228} &\bf{1.8365} &\bf{1.8457} &\bf{1.8512} &\bf{1.8537} & \bf{1.8542}\\
Greedy&\bf{1.6814}&\bf{1.7387}&\bf{1.7766}&\bf{1.8034}&\bf{1.8228}&\bf{1.8365}&\bf{1.8457}&\bf{1.8512}&\bf{1.8537} &\bf{1.8542}\\
\hline
win/tie/loss&5/45/0&3/47/0&4/46/0&4/46/0&\bf{7/43/0}&5/45/0&\bf{7/43/0}&\bf{6/44/0}&5/45/0&\bf{6/44/0}\\
\hline
\end{tabular}}
\end{center}
\end{table*}

\subsection{Maximizing Weakly Monotone and Strongly Submodular Functions}\label{subsec-exp-2}

For the problem class of maximizing weakly monotone and strongly submodular functions in Definition~\ref{def:prob-weak}, we consider the two applications of search-and-tracking and recommender systems, introduced in~\citep{bernardini2020through}. We compare the GSEMO with the previous best algorithm, i.e., the generalized greedy algorithm in Algorithm~\ref{alg:general-greedy}, as well as the standard greedy algorithm in Algorithm~\ref{alg:greedy}. The number $T$ of iterations of the GSEMO is set to $2ek^2(k+1)n$ as suggested by Theorem~\ref{theo-weak}.

\subsubsection{Search-and-Tracking}

The task of search-and-tracking~\citep{bernardini2016leveraging,piacentini2019autonomous} is to use a UAV to detect a moving target by executing a sequence of candidate flight search patterns. In particular, it is to select a subset from the set $V$ of all candidate patterns and perform the selected patterns in some specific order, such that the detection time is minimized. Here one item corresponds to one search pattern.

Let $\Gamma$ denote the set of all possible paths that the target travels along. Each search pattern $\sigma \in V$ is associated with a subset $\Gamma_{\sigma}$ of paths (i.e., $\Gamma_{\sigma} \subseteq \Gamma$), a time stamp $t_{\sigma}$ and a detection probability~$p_{\sigma}$. That is, if the UAV executes a pattern $\sigma$ at time $t_{\sigma}$ and the target takes a path in $\Gamma_{\sigma}$, the UAV can detect the target with probability $p_{\sigma}$; otherwise, the target cannot be detected. Given a sequence $s=(s_1,s_2,\ldots,s_l)$ of patterns and the path $\gamma$ that the target takes, we use a random variable $\omega_i$ to denote whether the target is detected by executing pattern $s_i$ at time $t_{s_i}$. If the target is detected, $\omega_i=1$, otherwise $\omega_i=0$. Then, the probability of observing $\omega_1,\omega_2,\ldots, \omega_l$ is
\begin{align}
P(\omega_1,\omega_2,\ldots, \omega_l \mid s,\gamma)=\prod_{i:\gamma \notin \Gamma_{s_i},\omega_i=1} (1-\omega_i) \cdot  \prod_{i:\gamma \in \Gamma_{s_i},\omega_i=1} p_{s_i} \cdot \prod_{i:\gamma \in \Gamma_{s_i},\omega_i=0} (1-p_{s_i}).
\end{align}
Assume that $\gamma$ satisfies the uniform distribution over $\Gamma$, implying
\begin{align}P(\omega_1,\omega_2,\ldots, \omega_l \mid s)=\sum_{\gamma \in \Gamma} \frac{1}{|\Gamma|}\cdot P(\omega_1,\omega_2,\ldots, \omega_l \mid s,\gamma).\end{align} The first detection time given $s=(s_1,s_2,\ldots,s_l)$ is defined as
\begin{align}
\tau_s=\begin{cases}
t_{s_k} & \text{if $\omega_k=1$ and $\forall i <k: \omega_i=0$}\\
K & \text{if $\forall i \in \{1,2,\ldots,l\}: \omega_i=0$}
\end{cases},
\end{align}
which is a random variable. Note that $K \geq \max_{\sigma \in V} t_{\sigma}$ is a penalty for failing to detect the target. Thus, the goal of search-and-tracking is to select a sequence $s=(s_1,s_2,\ldots,s_l)$ of search patterns minimizing the expected first detection time
\begin{equation}
\begin{aligned}
\mathbb{E}[\tau_s]&=\sum^{l}_{k=1} t_{s_k} \cdot P(\omega_1=0,\ldots,\omega_{k-1}=0,\omega_k=1 \mid s) + K \cdot P(\omega_1=0,\ldots,\omega_{l}=0 \mid s)\\
&= \sum^{l}_{k=1} t_{s_k} \cdot \left(P(\omega_1=0,\ldots,\omega_{k-1}=0 \mid s)-P(\omega_1=0,\ldots,\omega_k=0 \mid s)\right) \\
&\quad + K \cdot P(\omega_1=0,\ldots,\omega_{l}=0 \mid s)\\
&=K-\sum^{l}_{k=1} (K-t_{s_k}) \cdot \left(P(\omega_1=0,\ldots,\omega_{k-1}=0 \mid s)-P(\omega_1=0,\ldots,\omega_k=0 \mid s)\right),
\end{aligned}
\end{equation}
which is equivalent to maximizing
\begin{align}
f(s)=\sum^{l}_{k=1} (K-t_{s_k}) \cdot \left(P(\omega_1=0,\ldots,\omega_{k-1}=0 \mid s)-P(\omega_1=0,\ldots,\omega_k=0 \mid s)\right).
\end{align}

We use the same experimental setting as~\citep{bernardini2020through}. There are 40 paths in total, 20 search patterns and the budget is 10. That is, $|\Gamma|=40$, $n=|V|=20$ and $k=10$. The penalty $K$ is set to $\max_{\sigma \in V} t_{\sigma}$. Denote $V$ as $\{\sigma_1,\sigma_2,\ldots,\sigma_n\}$. For any $i$, the detection probability $p_{\sigma_i}$ of $\sigma_i$ is set as $\min\{\max\{mi/(n-1)+1/2-m/2,0.001\},0.999\}$. $m$ is set to $\{-1,-0.8,\ldots,1\}$. For each $m$, we generate 1000 random problem instances: for any $i$, the subset $\Gamma_{\sigma_i}$ of paths is randomly sampled from $\Gamma$; the time stamp $t_{\sigma_1}=1+r$ and $\forall i >1: t_{\sigma_i}=t_{\sigma_{i-1}}+r$, where $r$ is independently randomly sampled from $[0,n)$.

The results are shown in Table~\ref{tab-weak-st-1}. We can observe that the GSEMO always achieves the best average objective value, and is always significantly better than the generalized greedy algorithm by the sign-test with confidence level 0.05. The greedy algorithm performs the worst, which is expected because it fails to achieve a constant approximation guarantee~\citep{bernardini2020through}. We can also observe that the gap between the generalized greedy algorithm and the greedy algorithm when $m>0$ is larger than that when $m<0$. This is consistent with the empirical observation in~\citep{bernardini2020through}. By the setting of detection probability $p_{\sigma_i}$, we know that $m>0$ implies the monotone increasing property of $p_{\sigma_i}$ w.r.t. $i$. Thus, the greedy algorithm will select the search patterns with a greater time stamp (implying a larger detection probability) first and cannot exploit those early patterns, because it can only add a search pattern to the end of the current sequence. When $m<0$, the detection probability $p_{\sigma_i}$ is monotone decreasing with $i$, and thus the greedy algorithm can exploit all patterns and performs relatively better. It can also be clearly observed that the objective value has the trend of decreasing with $m$. This is because as $m$ increases, the detection probability of the first half of search patterns in $V$ will decrease, while that of the second half will increase. Thus, the expected detection time $\mathbb{E}[\tau_s]$ will increase, implying that the objective $f(s)=K-\mathbb{E}[\tau_s]$ will decrease.

\begin{table*}[ht!]\centering\caption{Comparison between the GSEMO, the GSEMO$_k$, the generalized greedy algorithm (denoted as G-Greedy) and the greedy algorithm for the application of search-and-tracking with $n=20$, $k=10$ and $m\in \{-1,-0.8,\ldots,1\}$. For each $m$, the average objective value over 1000 random problem instances is reported, and the largest one is bolded. The count of win/tie/loss denotes the number of problem instances where the GSEMO has a larger/equal/smaller objective value than the generalized greedy algorithm or the greedy algorithm, and significant cells by the sign-test with confidence level $0.05$ are bolded. OPT denotes the average optimal objective value.}\label{tab-weak-st-1}\vspace{0.5em}
\scriptsize
\newsavebox{\tablebox}
\begin{lrbox}{\tablebox}
\setlength{\tabcolsep}{0.5mm}{
\begin{tabular}{c|ccccccccccc}
\hline
$m$&$-1$&$-0.8$&$-0.6$&$-0.4$&$-0.2$&$0$&$0.2$&$0.4$&$0.6$&$0.8$&$1$\\
\hline
OPT & 167.933 & 165.275 &164.389 &161.643 &157.801 &152.727 &148.101 &138.372 &125.257 &103.913 &84.119 \\
\hline
GSEMO& \bf{167.933} & \bf{165.275} & \bf{164.389}  & \bf{161.643} & \bf{157.801} & \bf{152.727} & \bf{148.101} & \bf{138.372} & \bf{125.257} & \bf{103.913} & \bf{84.119}\\
GSEMO$_k$& \bf{167.933} & \bf{165.275} & 164.388 & 161.642 & 157.800 & \bf{152.727} & 148.100 & \bf{138.372} & \bf{125.257} & \bf{103.913} & \bf{84.119}\\
G-Greedy ($\bullet$)& 167.926 & 165.264 & 164.376 & 161.631 & 157.788 & 152.716 & 148.088 & 138.360 & 125.236 & 103.893 & 84.108\\
Greedy ($\circ$)& 150.361 & 142.676 & 138.445 & 134.491 & 127.931 & 121.020 & 115.887 & 105.177 & 90.598 & 69.076 & 53.040\\
\hline
win/tie/loss ($\bullet$)& \bf{28/972/0} & \bf{36/963/1} & \bf{41/959/0} & \bf{39/961/0} & \bf{39/961/0} & \bf{39/956/5} & \bf{44/956/0} & \bf{57/943/0} & \bf{88/912/0} & \bf{92/908/0} & \bf{65/935/0}\\
win/tie/loss ($\circ$)& \bf{881/119/0} & \bf{896/104/0} & \bf{892/108/0} & \bf{903/97/0} & \bf{892/108/0} & \bf{901/98/1} & \bf{902/98/0} & \bf{903/97/0} & \bf{910/90/0} & \bf{935/65/0} & \bf{968/32/0}\\
\hline
\end{tabular}}
\end{lrbox}
\scalebox{0.9}{\usebox{\tablebox}}
\end{table*}

Table~\ref{tab-weak-st-2} shows the results for a larger problem scale, where $n=50$ and $k=20$. The setting is the same as above except $|\Gamma|=80$. We can have a similar observation as that for $n=20$ and $k=10$.

\begin{table*}[h!]\centering\caption{Comparison between the GSEMO, the GSEMO$_k$, the generalized greedy algorithm (denoted as G-Greedy) and the greedy algorithm for the application of search-and-tracking with $n=50$, $k=20$ and $m\in \{-1,-0.8,\ldots,1\}$. For each $m$, the average objective value over 1000 random problem instances is reported, and the largest one is bolded. The count of win/tie/loss denotes the number of problem instances where the GSEMO has a larger/equal/smaller objective value than the generalized greedy algorithm or the greedy algorithm, and significant cells by the sign-test with confidence level $0.05$ are bolded.}\label{tab-weak-st-2}\vspace{0.5em}
\scriptsize
\begin{lrbox}{\tablebox}
\setlength{\tabcolsep}{0.5mm}{
\begin{tabular}{c|ccccccccccc}
\hline
$m$&$-1$&$-0.8$&$-0.6$&$-0.4$&$-0.2$&$0$&$0.2$&$0.4$&$0.6$&$0.8$&$1$\\
\hline
GSEMO& \bf{1035.408} & \bf{1033.593} & \bf{1028.054} & \bf{1026.896} & \bf{1027.760} & \bf{1036.752} & \bf{1021.015} & \bf{1004.615} & \bf{966.179} & \bf{895.946} & \bf{802.434} \\
GSEMO$_k$& \bf{1035.408} & \bf{1033.593} & 1028.053 & \bf{1026.896} & \bf{1027.760} & 1036.751 & \bf{1021.015} & \bf{1004.615} & \bf{966.179} & \bf{895.946} & \bf{802.434} \\
G-Greedy ($\bullet$)& 1035.397 & 1033.576 & 1028.039 & 1026.882 & 1027.738 & 1036.730 & 1020.986 & 1004.594 & 966.156 & 895.924 & 802.414\\
Greedy ($\circ$)& 982.784 & 950.218 & 917.594 & 889.992 & 852.750 & 836.899 & 761.605 & 643.897 & 494.501 & 425.303  & 375.139\\
\hline
win/tie/loss ($\bullet$)& \bf{104/892/4} & \bf{151/849/0} & \bf{146/854/0} & \bf{172/828/0} & \bf{179/821/0} & \bf{149/851/0} & \bf{137/863/0} & \bf{101/898/1} & \bf{68/932/0} & \bf{39/961/0} & \bf{34/965/1}\\
win/tie/loss ($\circ$)& \bf{967/33/0} & \bf{970/30/0} & \bf{975/25/0} & \bf{970/30/0} & \bf{974/26/0} & \bf{964/36/0} & \bf{973/27/0} & \bf{996/4/0} & \bf{1000/0/0} & \bf{1000/0/0} & \bf{1000/0/0}\\
\hline
\end{tabular}}
\end{lrbox}
\scalebox{0.88}{\usebox{\tablebox}}
\end{table*}

\subsubsection{Recommender Systems}

The application of recommender systems~\citep{bernardini2020through} is to select a sequence of movies to maximize the satisfaction of a user. Given a set $V$ of movies, let the function $g: V \rightarrow [0,1]$ represent the satisfaction probability of a user to a movie. Let $\mathcal{T}$ denote the set of topics, and let $p_{\sigma}(t)$ denote to which extent a movie $\sigma$ covers a topic $t$. Given a sequence $s=(s_1,s_2,\ldots,s_l)$ of movies, a user will select a topic~$t$ from $\mathcal{T}$ uniformly at random, and the corresponding movie in this sequence will be selected as follows: $s_1$ will be selected with probability $p_{s_1}(t)$; if $s_1$ is not selected (occurring with probability $1-p_{s_1}(t)$), $s_2$ will be selected with probability $p_{s_2}(t)$; this process will be continued until a movie is successfully selected. Thus, the satisfaction of the user on the selected topic $t$ is
\begin{align}
\sum^l_{i=1} g(s_i) \cdot \prod^{i-1}_{j=1}(1-p_{s_j}(t))\cdot p_{s_i}(t),
\end{align}
and the objective function
\begin{align}
f(s)=\sum_{t\in \mathcal{T}}\frac{1}{|\mathcal{T}|}\cdot \sum^l_{i=1} g(s_i) \cdot \prod^{i-1}_{j=1}(1-p_{s_j}(t))\cdot p_{s_i}(t).
\end{align}
Here one item corresponds to one movie.

We first fix the number $n=|V|$ of movies to be 500 and vary the budget $k$ in $\{10,12,\ldots,30\}$. The number $|\mathcal{T}|$ of topics is set to 50. For each $k$, we randomly generate 50 problem instances. That is, for each movie in $V$, the satisfaction probability of the user is randomly selected from $[0,1)$, and the extent of the movie covering each topic is also independently randomly sampled from $[0,1)$. Table~\ref{tab-weak-rs-1} shows the results. Though the average objective values obtained by the GSEMO and the generalized greedy algorithm cannot be distinguished when they are rounded to four decimal points, the number of win/tie/loss still shows that for each $k$, the GSEMO is never worse on all 50 problem instances, and is even significantly better by the sign-test with confidence level 0.05. As expected, the greedy algorithm performs the worst. It seems strange that the objective value obtained by the greedy algorithm keeps unchanged when $k$ increases. In fact, it increases, but the increment is very small and can be neglected when rounded to four decimal points.

\begin{table*}[h!]\caption{Comparison between the GSEMO, the GSEMO$_k$, the generalized greedy algorithm (denoted as G-Greedy) and the greedy algorithm for the application of recommender systems with $n=500$ and $k \in \{10,12,\ldots,30\}$. For each $k$, the average objective value over 50 random problem instances is reported, and the largest one is bolded. The count of win/tie/loss denotes the number of problem instances where the GSEMO has a larger/equal/smaller objective value than the generalized greedy algorithm or the greedy algorithm, and significant cells by the sign-test with confidence level $0.05$ are bolded.}\label{tab-weak-rs-1}
\scriptsize
\begin{center}
\setlength{\tabcolsep}{0.75mm}{
\begin{tabular}{c|ccccccccccc}
\hline
$k$ & $10$ &$12$&$14$ & $16$ & $18$ & $20$ &$22$ &$24$&$26$ & $28$ & $30$\\
\hline
GSEMO& \bf{0.9892} & \bf{0.9896} & \bf{0.9898} & \bf{0.9899} & \bf{0.9899} & \bf{0.9899} & \bf{0.9899} & \bf{0.9900} & \bf{0.9900} & \bf{0.9900} & \bf{0.9900} \\
GSEMO$_k$& \bf{0.9892} & \bf{0.9896} & \bf{0.9898} & \bf{0.9899} & \bf{0.9899} & \bf{0.9899} & \bf{0.9899} & \bf{0.9900} & \bf{0.9900} & \bf{0.9900} & \bf{0.9900} \\
G-Greedy ($\bullet$)& 0.9891 & \bf{0.9896} & \bf{0.9898} & \bf{0.9899} & \bf{0.9899} & \bf{0.9899} & \bf{0.9899} & \bf{0.9900} & \bf{0.9900} & \bf{0.9900} & \bf{0.9900} \\
Greedy ($\circ$)& 0.9631 & 0.9631 & 0.9631 & 0.9631 & 0.9631 & 0.9631 & 0.9631 & 0.9631 & 0.9631 & 0.9631 & 0.9631\\
\hline
win/tie/loss ($\bullet$)&\bf{47/3/0}&\bf{46/4/0}&\bf{46/4/0}&\bf{46/4/0}&\bf{44/6/0}&\bf{37/13/0}&\bf{31/19/0}&\bf{27/23/0} &\bf{21/29/0} &\bf{18/32/0}&\bf{16/34/0}\\
win/tie/loss ($\circ$)&\bf{50/0/0}&\bf{50/0/0}&\bf{50/0/0}&\bf{50/0/0}&\bf{50/0/0}&\bf{50/0/0}&\bf{50/0/0}&\bf{50/0/0} &\bf{50/0/0} &\bf{50/0/0} & \bf{50/0/0}\\
\hline
\end{tabular}}
\end{center}
\end{table*}

We next fix $k=20$ and vary $n \in \{100,200,\ldots,1000\}$. The setting is the same as above. The results in Table~\ref{tab-weak-rs-2} are similar to that we have observed in Table~\ref{tab-weak-rs-1}. That is, the GSEMO is always significantly better than the generalized greedy algorithm by the sign-test with confidence level $0.05$, and the greedy algorithm is the worst.

\begin{table*}[h!]\caption{Comparison between the GSEMO, the GSEMO$_k$, the generalized greedy algorithm (denoted as G-Greedy) and the greedy algorithm for the application of recommender systems with $n\in\{100,200,\ldots,1000\}$ and $k=20$. For each $n$, the average objective value over 50 random problem instances is reported, and the largest one is bolded. The count of win/tie/loss denotes the number of problem instances where the GSEMO has a larger/equal/smaller objective value than the generalized greedy algorithm or the greedy algorithm, and significant cells by the sign-test with confidence level $0.05$ are bolded.}\label{tab-weak-rs-2}
\scriptsize
\begin{center}
\setlength{\tabcolsep}{0.75mm}{
\begin{tabular}{c|cccccccccc}
\hline
$n$ & $100$ & $200$ & $300$ & $400$ & $500$ & $600$ & $700$ & $800$ & $900$ & $1000$\\
\hline
GSEMO& \bf{0.9830} & \bf{0.9881} & \bf{0.9894} & \bf{0.9899} & \bf{0.9899} & \bf{0.9899} & \bf{0.9900}  & \bf{0.9900}  & \bf{0.9900}  & \bf{0.9900}\\
GSEMO$_k$& \bf{0.9830} & \bf{0.9881} & \bf{0.9894} & \bf{0.9899} & \bf{0.9899} & \bf{0.9899} & \bf{0.9900} & \bf{0.9900} & \bf{0.9900} & \bf{0.9900}\\
G-Greedy ($\bullet$)& \bf{0.9830} & \bf{0.9881} & 0.9893 & \bf{0.9899} & \bf{0.9899} & \bf{0.9899} & \bf{0.9900}  & \bf{0.9900}  & \bf{0.9900}  & \bf{0.9900}\\
Greedy ($\circ$)& 0.9431 & 0.9543& 0.9605 & 0.9636 & 0.9631& 0.9614 & 0.9664 & 0.9663 & 0.9675 & 0.9681\\
\hline
win/tie/loss ($\bullet$)&\bf{47/3/0}&\bf{44/6/0}&\bf{45/5/0}&\bf{41/9/0}&\bf{37/13/0}&\bf{37/13/0}&\bf{22/28/0}&\bf{19/31/0} &\bf{12/38/0} &\bf{8/42/0}\\
win/tie/loss ($\circ$)&\bf{50/0/0}&\bf{50/0/0}&\bf{50/0/0}&\bf{50/0/0}&\bf{50/0/0}&\bf{50/0/0}&\bf{50/0/0}&\bf{50/0/0} &\bf{50/0/0} &\bf{50/0/0}\\
\hline
\end{tabular}}
\end{center}
\end{table*}

\subsection{Maximizing DAG Monotone Submodular Functions}\label{subsec-exp-3}

For the problem class of maximizing DAG monotone submodular functions in Definition~\ref{def:prob-dag}, we consider the two applications where each edge $(v_i,v_j)$ on the directed acyclic graph $G=(V,E)$ has a weight $w_{i,j}$, and the function $h: 2^{E} \rightarrow \mathbb{R}$ defined over the set $E$ of edges satisfies that \begin{align}\label{eq-modular}\forall X \subseteq E: h(X)=\sum_{(v_i,v_j) \in X} w_{i,j}\end{align} and \begin{align}\label{eq-submodular}\forall X \subseteq E: h(X)=\sum_{v_j \in V(X)} \left(1-\prod_{(v_i,v_j) \in X} \left(1-w_{i,j}\right)\right),\end{align} respectively. Note that $V(X)$ denotes the item set covered by the edge set $X \subseteq E$. These two $h$ functions are modular and submodular, respectively. For each application, we use a synthetic data set and a real-world data set to compare the GSEMO with the previous best algorithm, i.e., the \textsc{OMegA} algorithm in Algorithm~\ref{alg:OMegA}, as well as the standard greedy algorithm in Algorithm~\ref{alg:greedy}. The number $T$ of iterations of the GSEMO is set to $4ek^2n^2$ as suggested by Theorem~\ref{theo-dag}.

\subsubsection{Modular $h$}

We first consider the case where $h$ satisfies Eq.~(\refeq{eq-modular}), which is modular. For the synthetic data set, we use the same setting as in~\citep{tschiatschek2017selecting}. The graph $G=(V,E)$ is constructed as follows: for each item $v_i \in V$, randomly select a subset of size $\min\{d,n-i\}$ from $\{v_{i+1},\ldots,v_n\}$ and set an edge from $v_i$ to each item in the selected subset and also to itself (i.e., a self-cycle). Each weight $w_{i,j}$ is randomly sampled from $[0,1]$. We set $n=|V|=30$, $k=5$, and use $d \in \{1,2,\ldots,10\}$. For each $d$, we randomly generate 50 problem instances, and report the number of instances where the GSEMO wins, ties or loses against the \textsc{OMegA} algorithm or the greedy algorithm.

The results are shown in Table~\ref{tab-dag-modular-1}. We can observe that the GSEMO is better than the \textsc{OMegA}, and is further significantly better by the sign-test with confidence level $0.05$. The greedy algorithm performs the worst, which is consistent with that it fails to achieve any constant approximation guarantee for DAG monotone submodular function maximization~\citep{tschiatschek2017selecting}. We can also observe that the objective value increases with $d$. This is expected because a larger $d$ implies a denser graph $G$, and thus a sequence will have a larger edge set, resulting in a larger objective value due to the monotonicity of $h$.

\begin{table*}[h!]\caption{Comparison between the GSEMO, the GSEMO$_k$, the \textsc{OMegA} and the greedy algorithm for the application of modular $h$ on the synthetic data set where $n=30$ and $k=5$. For each $d\in \{1,2,\ldots,10\}$, the average objective value over 50 random problem instances is reported, and the largest one is bolded. The count of win/tie/loss denotes the number of problem instances where the GSEMO has a larger/equal/smaller objective value than the \textsc{OMegA} or the greedy algorithm, and significant cells by the sign-test with confidence level $0.05$ are bolded. OPT denotes the average optimal objective value.}\label{tab-dag-modular-1}
\scriptsize
\begin{center}
\setlength{\tabcolsep}{0.75mm}{
\begin{tabular}{c|cccccccccc}
\hline
$d$ & $1$ &$2$ &$3$ & $4$ & $5$ & $6$ &$7$ & $8$ & $9$ & $10$\\
\hline
OPT &6.1201	&7.2197	&8.0528	&8.5084	&9.2098	&9.5183	&9.9011	&10.0722	&10.1404	&10.5540\\
\hline
GSEMO& \bf{6.1201} & \bf{7.1984} & \bf{8.0428} & \bf{8.4717} & \bf{9.1984} & \bf{9.5120} & \bf{9.8789} & \bf{10.0651} & \bf{10.1227} & \bf{10.5461} \\
GSEMO$_k$& 6.0972 & 7.1389 & 7.9646 & 8.3636 & 9.0862 & 9.3991 & 9.8078 & 9.9829 & 10.0287 & 10.4476 \\
\textsc{OMegA} ($\bullet$) & 5.9891 & 6.9784 & 7.7666 & 8.1206 & 8.7551 & 8.9992 & 9.2882 & 9.6966 & 9.6631 & 10.0402\\
Greedy ($\circ$)& 5.1318  & 5.7404  & 6.3720  & 6.6858 & 6.6792 & 7.2544 & 7.3824 & 7.6038 & 7.9989 & 7.6820\\
\hline
win/tie/loss ($\bullet$)&\bf{24/26/0}&\bf{35/15/0}&\bf{32/18/0}&\bf{33/17/0}&\bf{31/19/0} &\bf{34/16/0}&\bf{36/14/0}&\bf{31/19/0} &\bf{36/14/0} & \bf{39/10/1}\\
win/tie/loss ($\circ$)&\bf{50/0/0}&\bf{50/0/0}&\bf{49/1/0}&\bf{50/0/0}&\bf{50/0/0}&\bf{50/0/0}&\bf{49/1/0}&\bf{50/0/0} &\bf{49/1/0} &\bf{50/0/0}\\
\hline
\end{tabular}}
\end{center}
\end{table*}

Next we use the real-world data set \textit{Movielens 1M}\footnote{\url{https://grouplens.org/datasets/movielens/1m/}} to construct a directed acyclic graph $G=(V,E)$ if not counting self-cycles. This data set contains 1,000,209 ratings made by 6040 users for 3706 movies. As in~\citep{mitrovic2018submodularity}, to make the data to be representative, we delete all users who have rated fewer than 20 movies or
more than 50 movies, and also delete all movies which have received fewer than 1000 ratings. The data set after preprocessing has 2047 users and 207 movies. We order these 207 movies according to their time stamps recorded in the data set. Each vertex in the graph $G$ corresponds to one movie. Thus, $n=|V|=207$. For any $i<j$, there is an edge between movie $v_i$ and movie $v_j$, and the weight $w_{i,j}$ is set as the conditional probability of rating $v_j$ given that a user has already rated $v_i$. That is, $w_{i,j}=N_{i,j}/(N_i+20)$, where $N_{i,j}$ is the number of users who have rated $v_i$ before $v_j$, $N_i$ is the number of users who have rated $v_i$, and $20$ in the denominator is used to avoid overfitting to rare events. For any $v_i$, there is a self-cycle, and the weight $w_{i,i}$ is set as the probability of rating $v_i$, i.e., $w_{i,i}=N_i/(2047+20)$, where $2047$ is the total number of users. We set the budget $k$ as $\{2,3,\ldots,10\}$. Thus, the task is to recommend a sequence of movies with length at most $k$ based on the watching histories of existing users.

The results are shown in Table~\ref{tab-dag-modular-2}. As the data set is fixed, the \textsc{OMegA} and the greedy algorithm have only one output objective value. But for the GSEMO, which is a randomized algorithm, we repeat its run ten times independently, and report the average objective value and the standard deviation. Note that the standard deviation of the GSEMO is always 0 in this experiment, i.e., the same good solutions are found in the ten independent runs, and thus we neglect it in Table~\ref{tab-dag-modular-2}. We can observe that the GSEMO and the \textsc{OMegA} achieve the same objective value, and are better than the greedy algorithm. The objective value obtained by these three algorithms increases with $k$ as expected.

\begin{table*}[h!]\caption{Comparison between the GSEMO, the GSEMO$_k$, the \textsc{OMegA} and the greedy algorithm for the application of modular $h$ on the real-world data set \textit{Movielens 1M} where $n=207$. For each $k\in \{2,3,\ldots,10\}$, the largest objective value is bolded. OPT denotes the optimal objective value. `--' means that no results were obtained after running several days.}\label{tab-dag-modular-2}
\scriptsize
\begin{center}
\setlength{\tabcolsep}{1.8mm}{
\begin{tabular}{c|ccccccccc}
\hline
$k$ &$2$ &$3$ & $4$ & $5$ & $6$ &$7$ & $8$ & $9$ & $10$\\
\hline
OPT &1.2719 &2.3223 &3.6766 &5.3801 & -- & -- & -- & -- & --\\
\hline
GSEMO& \bf{1.2719} & \bf{2.3223} & \bf{3.6766} & \bf{5.3801} & \bf{7.3126} & \bf{9.4379} & \bf{11.8124} & \bf{14.1536} & \bf{16.5076}\\
GSEMO$_k$& \bf{1.2719} & \bf{2.3223} & \bf{3.6766} & \bf{5.3801} & \bf{7.3126} & \bf{9.4379} & \bf{11.8124} & \bf{14.1536} & \bf{16.5076}\\
\textsc{OMegA} & \bf{1.2719} & \bf{2.3223} & \bf{3.6766} & \bf{5.3801} & \bf{7.3126} & \bf{9.4379} & \bf{11.8124} & \bf{14.1536} & \bf{16.5076}\\
Greedy & 1.2678 & 2.1084 & 3.0038 & 4.1056 & 5.0370 & 6.0003 & 6.9073 & 7.9001 & 8.7954\\
\hline
\end{tabular}}
\end{center}
\end{table*}

\subsubsection{Submodular $h$}

We also compare their performance where $h$ satisfies Eq.~(\refeq{eq-submodular}), which is submodular. Table~\ref{tab-dag-submodular-1} shows the results on the synthetic data set. Note that compared with the modular case, there is one difference in the construction of the graph. That is, each weight $w_{i,i}$ is randomly sampled from $[0,0.1]$ instead of $[0,1]$, which is same as the setting in~\citep{tschiatschek2017selecting}. The results are similar to that we have observed in Table~\ref{tab-dag-modular-1} for the modular case. The GSEMO performs the best, and is significantly better than the \textsc{OMegA}. The greedy algorithm performs the worst.

\begin{table*}[h!]\caption{Comparison between the GSEMO, the GSEMO$_k$, the \textsc{OMegA} and the greedy algorithm for the application of submodular $h$ on the synthetic data set where $n=30$ and $k=5$. For each $d\in \{1,2,\ldots,10\}$, the average objective value over 50 random problem instances is reported, and the largest one is bolded. The count of win/tie/loss denotes the number of problem instances where the GSEMO has a larger/equal/smaller objective value than the \textsc{OMegA} or the greedy algorithm, and significant cells by the sign-test with confidence level $0.05$ are bolded. OPT denotes the average optimal objective value.}\label{tab-dag-submodular-1}
\scriptsize
\begin{center}
\setlength{\tabcolsep}{0.75mm}{
\begin{tabular}{c|cccccccccc}
\hline
$d$ & $1$ &$2$ &$3$ & $4$ & $5$ & $6$ &$7$ & $8$ & $9$ & $10$\\
\hline
OPT &2.9200 &3.5572 &3.7238 &3.8358 &3.8918 &3.9283 &3.9648 &3.9643 &3.9821 &4.0030\\
\hline
GSEMO& \bf{2.9200} & \bf{3.5293} & \bf{3.7066} & \bf{3.8202} & \bf{3.8809} & \bf{3.9214} & \bf{3.9435} & \bf{3.9574} & \bf{3.9730} & \bf{3.9965}\\
GSEMO$_k$& 2.8605 & 3.4670 & 3.6849 & 3.7976 & 3.8515 & 3.9021 & 3.9390 & 3.9330 & 3.9716 & 3.9895\\
\textsc{OMegA} ($\bullet$)& 2.6880 & 3.2915 & 3.5747 & 3.6902 & 3.7791  & 3.8499  & 3.8793 & 3.8788 & 3.9360 & 3.9276 \\
Greedy ($\circ$)& 1.7974 & 2.6336 & 2.8417 & 3.1242 & 3.1819 & 3.2774 & 3.1947 & 3.3560  & 3.4460  & 3.5317\\
\hline
win/tie/loss ($\bullet$)&\bf{35/15/0}&\bf{34/15/1}&\bf{34/14/2}&\bf{30/20/0}&\bf{30/19/1} &\bf{28/21/1}&\bf{29/20/1}&\bf{35/14/1} &\bf{32/18/0} & \bf{30/18/2}\\
win/tie/loss ($\circ$)&\bf{49/1/0}&\bf{48/2/0}&\bf{49/0/1}&\bf{49/0/1}&\bf{50/0/0}&\bf{49/0/1}&\bf{49/1/0}&\bf{50/0/0} &\bf{46/4/0} &\bf{49/1/0} \\
\hline
\end{tabular}}
\end{center}
\end{table*}

Table~\ref{tab-dag-submodular-2} shows the results on the real-world data set \textit{Movielens 1M}. The setting is same as that for the modular case. The standard deviation of the GSEMO is still 0 here, but the GSEMO is better than the \textsc{OMegA} now. Note that the GSEMO does not beat the \textsc{OMegA} in Table~\ref{tab-dag-modular-2} for the modular case, where they achieve the same objective value. The reason may be because the problem with modular $h$ is easier than that with submodular $h$, and thus the \textsc{OMegA} has already performed quite well. This will be verified in the next subsection, where we will show that the \textsc{OMegA} has achieved the optimal objective value for the modular case.

\begin{table*}[h!]\caption{Comparison between the GSEMO, the GSEMO$_k$, the \textsc{OMegA} and the greedy algorithm for the application of submodular $h$ on the real-world data set \textit{Movielens 1M} where $n=207$. For each $k\in \{2,3,\ldots,10\}$, the largest objective value is bolded. OPT denotes the optimal objective value. `--' means that no results were obtained after running several days.}\label{tab-dag-submodular-2}
\scriptsize
\begin{center}
\setlength{\tabcolsep}{1.8mm}{
\begin{tabular}{c|ccccccccc}
\hline
$k$ &$2$ &$3$ & $4$ & $5$ & $6$ &$7$ & $8$ & $9$ & $10$\\
\hline
OPT &1.1346 &1.8174 &2.6083 & 3.4869 & -- & -- & -- & -- & -- \\
\hline
GSEMO & \bf{1.1346} & \bf{1.8174} & \bf{2.6083} & \bf{3.4869} & \bf{4.3799}  & \bf{5.2848}  & \bf{6.1630}  & \bf{7.0603}  & \bf{7.9613} \\
GSEMO$_k$ & \bf{1.1346} & \bf{1.8174} & \bf{2.6083} & \bf{3.4869} & \bf{4.3799} & \bf{5.2848} & \bf{6.1630} & \bf{7.0603} & \bf{7.9613} \\
\textsc{OMegA} & \bf{1.1346}& 1.7718 & 2.5154 & 3.3350 & 4.1930 & 5.0884 & 6.0177 & 6.9551 & 7.8887\\
Greedy & \bf{1.1346} & 1.7718 & 2.4142 & 3.1357 & 3.7745 & 4.4225 & 5.0599 & 5.7298 & 6.4370 \\
\hline
\end{tabular}}
\end{center}
\end{table*}

\subsection{Discussion}\label{subsec-discussion}

The above experimental results have shown that the objective function value achieved by the GSEMO is at least as good as that achieved by the previous best algorithm for each considered problem class of maximizing monotone submodular functions over sequences, i.e., the greedy algorithm for maximizing prefix monotone submodular functions, the generalized greedy algorithm for maximizing weakly monotone and strongly submodular functions, and the \textsc{OMegA} algorithm for maximizing DAG monotone submodular functions.

However, we also note that the improvement can be quite different, depending on concrete applications and data sets. For example, for maximizing DAG monotone submodular functions on the real-world data set \textit{Movielens 1M}, when $h$ is modular, the GSEMO and the \textsc{OMegA} achieve the same objective values as shown in Table~\ref{tab-dag-modular-2}; while $h$ is submodular, the GSEMO achieves much larger objective values than the \textsc{OMegA} as shown in Table~\ref{tab-dag-submodular-2}. The reason may be because the previous algorithm has performed very well in some situations, and thus the GSEMO can bring very small or even no improvement. To validate this explanation, we compute the optimum (denoted as OPT) using exhaustive enumeration. Note that due to the computation time limit, we can only compute OPT in Tables~\ref{tab-weak-st-1}, and~\ref{tab-dag-modular-1} to~\ref{tab-dag-submodular-2}, and for Tables~\ref{tab-dag-modular-2} and~\ref{tab-dag-submodular-2}, OPT can be computed only for $k=2,3,4,5$.

From Tables~\ref{tab-dag-modular-2} and~\ref{tab-dag-submodular-2}, we can observe that when the GSEMO and the \textsc{OMegA} achieve the same objective value, the \textsc{OMegA} has already achieved OPT. In Table~\ref{tab-weak-st-1}, the GSEMO is slightly better than the generalized greedy algorithm, and we can observe that the objective value achieved by the generalized greedy algorithm has already been very close to OPT. In Tables~\ref{tab-dag-modular-1},~\ref{tab-dag-submodular-1} and~\ref{tab-dag-submodular-2} (except $k=2$), the GSEMO is much better than the \textsc{OMegA}, and we can observe that the objective value achieved by the \textsc{OMegA} has a relatively large gap to OPT. For Tables~\ref{tab-weak-st-1},~\ref{tab-dag-modular-1} and~\ref{tab-dag-submodular-1}, we also compute the approximation ratios of each algorithm to OPT, which are shown in Tables~\ref{tab-ratio-1} to~\ref{tab-ratio-3}, respectively. These results disclose that the GSEMO can make a large improvement when the performance of the previous algorithm is not close to OPT (e.g., in Tables~\ref{tab-ratio-2} and~\ref{tab-ratio-3}), and can still bring an improvement even when the previous algorithm has been nearly optimal (e.g., in Table~\ref{tab-ratio-1}).

\begin{table*}[h!]\caption{The approximation rato of the GSEMO, the generalized greedy algorithm (denoted as G-Greedy) and the greedy algorithm to OPT for the application of search-and-tracking where $n=20$, $k=10$ and $m\in \{-1,-0.8,\ldots,1\}$.}\label{tab-ratio-1}
\scriptsize
\begin{center}
\setlength{\tabcolsep}{1.5mm}{
\begin{tabular}{c|ccccccccccc}
\hline
$m$&$-1$&$-0.8$&$-0.6$&$-0.4$&$-0.2$&$0$&$0.2$&$0.4$&$0.6$&$0.8$&$1$\\
\hline
GSEMO/OPT &1 &1 &1 &1 &1 &1 &1 &1 &1 &1 &1\\
G-Greedy/OPT &0.9999 &0.9999 &0.9999 &0.9999 &0.9999 &0.9999 &0.9999 &0.9999 &0.9998 &0.9998 &0.9999\\
Greedy/OPT &0.8914 &0.8586 &0.8364 &0.8251 &0.8019 &0.7826 &0.7703 &0.7429 &0.7015 &0.6507 &0.6293\\
\hline
\end{tabular}}
\end{center}\vspace{-1.5em}
\end{table*}

\begin{table*}[h!]\caption{The approximation rato of the GSEMO, the \textsc{OMegA} and the greedy algorithm to OPT for the application of modular $h$ on the synthetic data set where $n=30$, $k=5$ and $d\in \{1,2,\ldots,10\}$.}\label{tab-ratio-2}
\scriptsize
\begin{center}
\setlength{\tabcolsep}{1.8mm}{
\begin{tabular}{c|cccccccccc}
\hline
$d$ & $1$ &$2$ &$3$ & $4$ & $5$ & $6$ &$7$ & $8$ & $9$ & $10$\\
\hline
GSEMO/OPT & 1 & 0.9972 & 0.9988 & 0.9959  & 0.9987 & 0.9993 & 0.9977 & 0.9993 & 0.9981 & 0.9992   \\
\textsc{OMegA}/OPT & 0.9789  & 0.9671 & 0.9642 & 0.9539  & 0.9496 & 0.9455 & 0.9379 & 0.9620 & 0.9531 & 0.9514  \\
Greedy/OPT & 0.8408  & 0.7969 & 0.7928 & 0.7858 & 0.7253 & 0.7636  & 0.7473 & 0.7562 & 0.7905 & 0.7292  \\
\hline
\end{tabular}}
\end{center}\vspace{-1.5em}
\end{table*}

\begin{table*}[h!]\caption{The approximation rato of the GSEMO, the \textsc{OMegA} and the greedy algorithm to OPT for the application of submodular $h$ on the synthetic data set where $n=30$, $k=5$ and $d\in \{1,2,\ldots,10\}$.}\label{tab-ratio-3}
\scriptsize
\begin{center}
\setlength{\tabcolsep}{1.8mm}{
\begin{tabular}{c|cccccccccc}
\hline
$d$ & $1$ &$2$ &$3$ & $4$ & $5$ & $6$ &$7$ & $8$ & $9$ & $10$\\
\hline
GSEMO/OPT & 1 & 0.9921 & 0.9952  & 0.9959 & 0.9972 & 0.9982 & 0.9946 & 0.9983 & 0.9977 & 0.9984\\
\textsc{OMegA}/OPT & 0.9226 & 0.9249  & 0.9596 & 0.9621 & 0.9708 & 0.9800 & 0.9784  & 0.9785 & 0.9884  & 0.9811  \\
Greedy/OPT & 0.6136  & 0.7383 & 0.7645 & 0.8149 & 0.8173 & 0.8335  & 0.8054 & 0.8468 & 0.8650 & 0.8822 \\
\hline
\end{tabular}}
\end{center}\vspace{-1em}
\end{table*}

Next, we examine the effectiveness of setting $f_1$ to $-\infty$ for $|s| \geq 2k$ in the bi-objective problem reformulation Eq.~(\refeq{def-CO-BO}). From the theoretical analyses in Sections~\ref{sec-theory1} to~\ref{sec-theory3}, we can find that only feasible sequences are required in the proofs. By setting $f_1=-\infty$ for $|s| > k$, the approximation guarantees derived in Theorems~\ref{theo-prefix},~\ref{theo-weak} and~\ref{theo-dag} still hold, but the upper bounds on the required expected number $\mathbb{E}[T]$ of iterations reduce from $2ek^2(k+1)n$, $2ek^2(k+1)n$ and $4ek^2n^2$ to $ek(k+1)^2n$, $ek(k+1)^2n$ and $2ek(k+1)n^2$, respectively, which is because the largest size $P_{\max}$ of the population during the run of GSEMO reduces from $2k$ to $k+1$. Our motivation of setting $f_1=-\infty$ for $|s| \geq 2k$ instead of $|s| > k$ is to allow infeasible sequences (i.e., sequences with $k<|s|<2k$) with small constraint violation degree to participate in the evolutionary process, which may bring performance improvement in practice.

Tables~\ref{tab-prefix-task-1} to~\ref{tab-dag-submodular-2} have shown the results of the GSEMO setting $f_1=-\infty$ for $|s| > k$, denoted as GSEMO$_k$. Note that as the setting for the GSEMO, the number of iterations of the GSEMO$_k$ has also been set to the theoretically required upper bound on the expected number of iterations, i.e., $ek(k+1)^2n$, $ek(k+1)^2n$ and $2ek(k+1)n^2$ for the three considered problem classes, respectively. We can observe that when the performance of the previous best algorithm is not close to OPT, e.g., in Tables~\ref{tab-dag-modular-1} and~\ref{tab-dag-submodular-1}, the GSEMO$_k$ is clearly better than the previous best algorithm, and the GSEMO can bring further improvement; when the previous best algorithm has performed very well, e.g., in Tables~\ref{tab-prefix-task-1} to~\ref{tab-weak-rs-2}, the performance of GSEMO$_k$ and GSEMO is almost the same, except that the GSEMO is slightly better in some cases (e.g., $k=10$ in Table~\ref{tab-prefix-task-1}). One may feel that the comparison between the GSEMO and GSEMO$_k$ is unfair, because the GSEMO runs for more iterations than the GSEMO$_k$. For the three considered problem classes, the GSEMO runs for $2ek^2(k+1)n$, $2ek^2(k+1)n$ and $4ek^2n^2$ iterations, respectively. However, as the GSEMO achieves a good performance quickly in practice (which will be shown later when we consider the running time), its final performance only decreases slightly and can be still clearly better than the performance of the GSEMO$_k$, even when the GSEMO employs the same number of iterations as the GSEMO$_{k}$. For example, for the problem of maximizing DAG monotone submodular functions on the synthetic data set, Figure~\ref{fig-constraint} plots the approximation ratio of the GSEMO, GSEMO$^*$ and GSEMO$_k$, where the GSEMO$^*$ denotes the GSEMO using the same number of iterations (i.e., $2ek(k+1)n^2$ iterations) as the GSEMO$_{k}$. We can observe that the GSEMO$^*$ is very close to the GSEMO, and is still clearly better than the GSEMO$_k$. Thus, these observations validate the effectiveness of setting $f_1=-\infty$ for $|s| \geq 2k$ instead of $|s| > k$.

\begin{figure*}[h!]\centering
\begin{minipage}[c]{0.42\linewidth}\centering
        \includegraphics[width=1\linewidth]{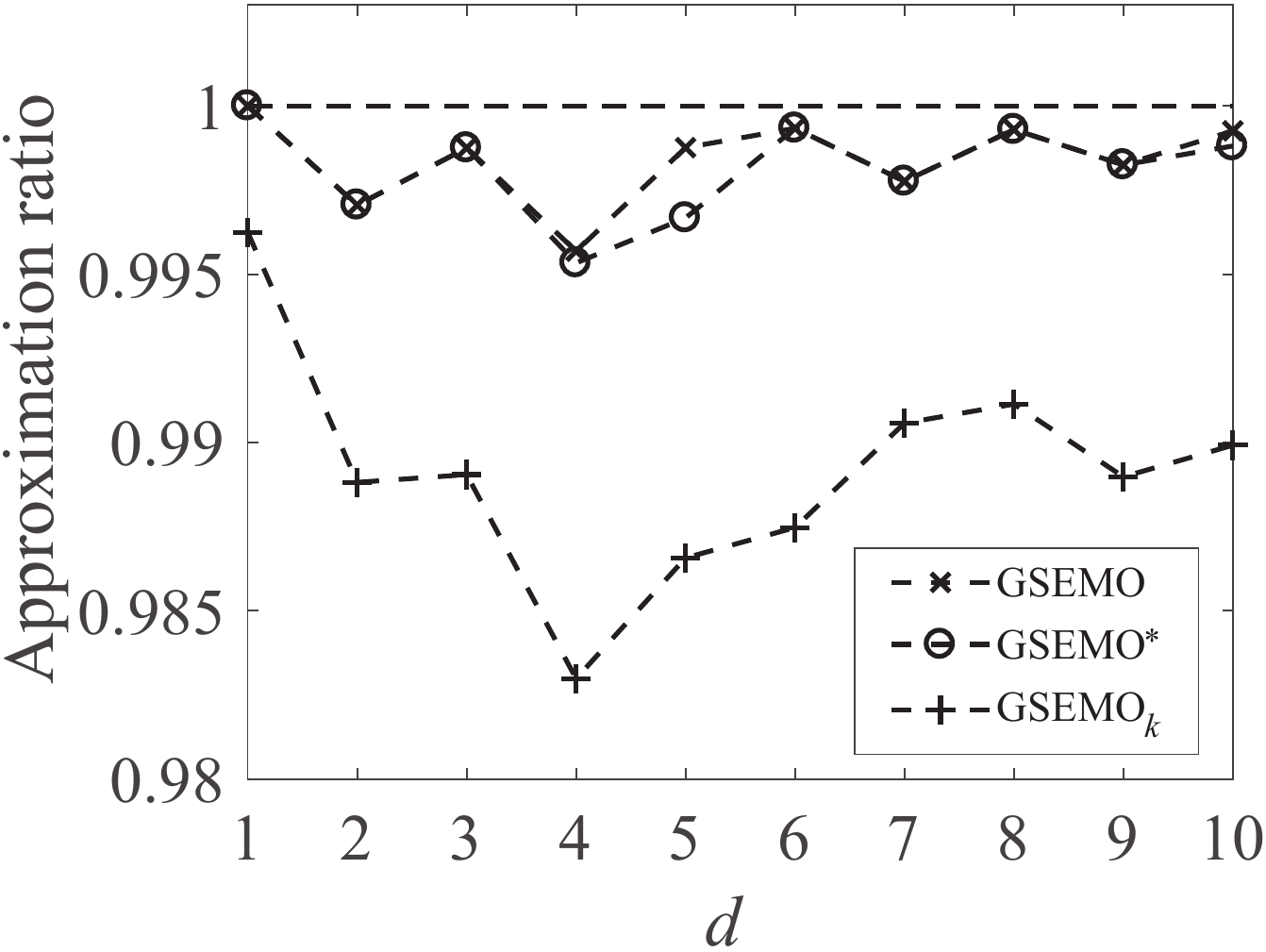}
\end{minipage}\hspace{2em}
\begin{minipage}[c]{0.42\linewidth}\centering
        \includegraphics[width=1\linewidth]{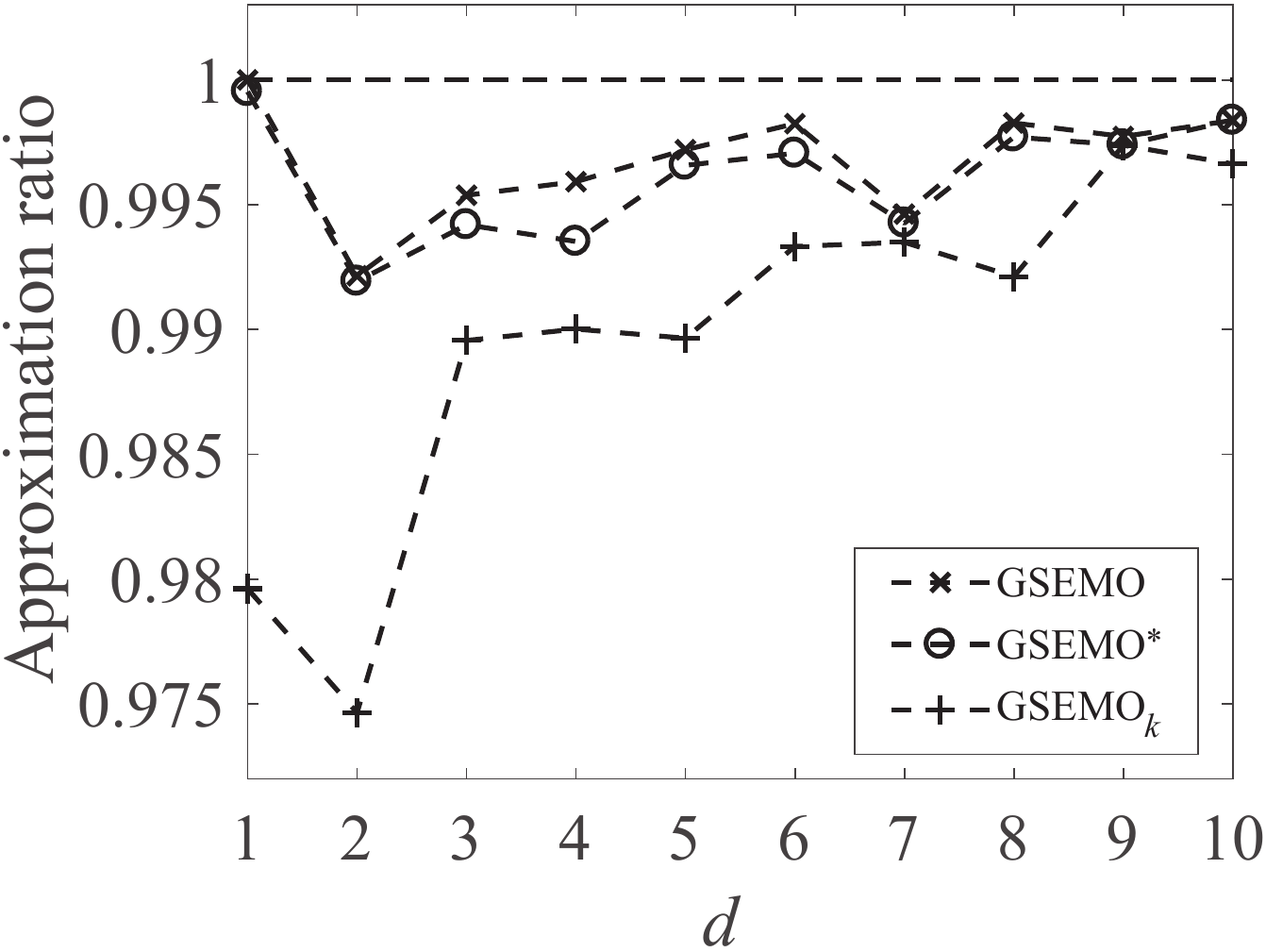}
\end{minipage}\\\vspace{0.3em}
\begin{minipage}[c]{0.42\linewidth}\centering
    \small(a) modular $h$
\end{minipage}\hspace{2em}
\begin{minipage}[c]{0.42\linewidth}\centering
    \small(b) submodular $h$
\end{minipage}
\caption{Approximation ratio of the GSEMO, GSEMO$^*$ and GSEMO$_k$ on maximizing DAG monotone submodular functions using the synthetic data set, where the GSEMO and GSEMO$^*$ set $f_1=-\infty$ for $|s| \geq 2k$, and the GSEMO$_k$ sets $f_1=-\infty$ for $|s| > k$; their number of iterations is set to $4ek^2n^2$, $2ek(k+1)n^2$ and $2ek(k+1)n^2$, respectively.}\label{fig-constraint}
\end{figure*}

Finally, we consider the running time, in the number of objective function evaluations. The greedy algorithm (i.e., Algorithm~\ref{alg:greedy}), the generalized greedy algorithm (i.e., Algorithm~\ref{alg:general-greedy}) and the \textsc{OMegA} (i.e., Algorithm~\ref{alg:OMegA}) take the time in the order of $kn$, $k^2n$ and $k\Delta|E|$ (where $\Delta$ is the smaller one between the largest indegree and outdegree of all items in the graph $G$, and $|E|$ is the number of edges of $G$), respectively. For the GSEMO in Algorithm~\ref{alg:POSeqSel}, the number $T$ of iterations has been set to the upper bound on its expected value derived in theoretical analysis, that is, $2ek^2(k+1)n$ for maximizing prefix monotone submodular functions, $2ek^2(k+1)n$ for maximizing weakly monotone and strongly submodular functions, and $4ek^2n^2$ for maximizing DAG monotone submodular functions.

We want to examine how efficient the GSEMO can be in practice. We take the problem of maximizing DAG monotone submodular functions on the synthetic data set with $d=5$ as an example. We plot the curve of the approximation ratio over the running time for the GSEMO, and select the greedy algorithm and the \textsc{OMegA} as the baselines. The curves are shown in Figure~\ref{fig-time}, where one unit on the $x$-axis corresponds to $kn$ objective evaluations. Compared with the theoretical running time $4ek^2n^2\approx 1631 kn$ (where $n=30$ and $k=5$ here) in Theorem~\ref{theo-dag}, we can observe that the GSEMO obtains a better performance much faster, implying that the GSEMO can be efficient in practice. This is expected, because the theoretical running time is the running time required by the GSEMO to achieve a good approximation in the worst case. We also plot the curve of the GSEMO$_k$, which is clearly below that of the GSEMO. This further confirms the advantage of utilizing infeasible sequences in the evolutionary process, i.e., setting $f_1$ to $-\infty$ for $|s| \geq 2k$ instead of $|s|>k$.

\begin{figure*}[h!]\centering
\begin{minipage}[c]{0.42\linewidth}\centering
        \includegraphics[width=1\linewidth]{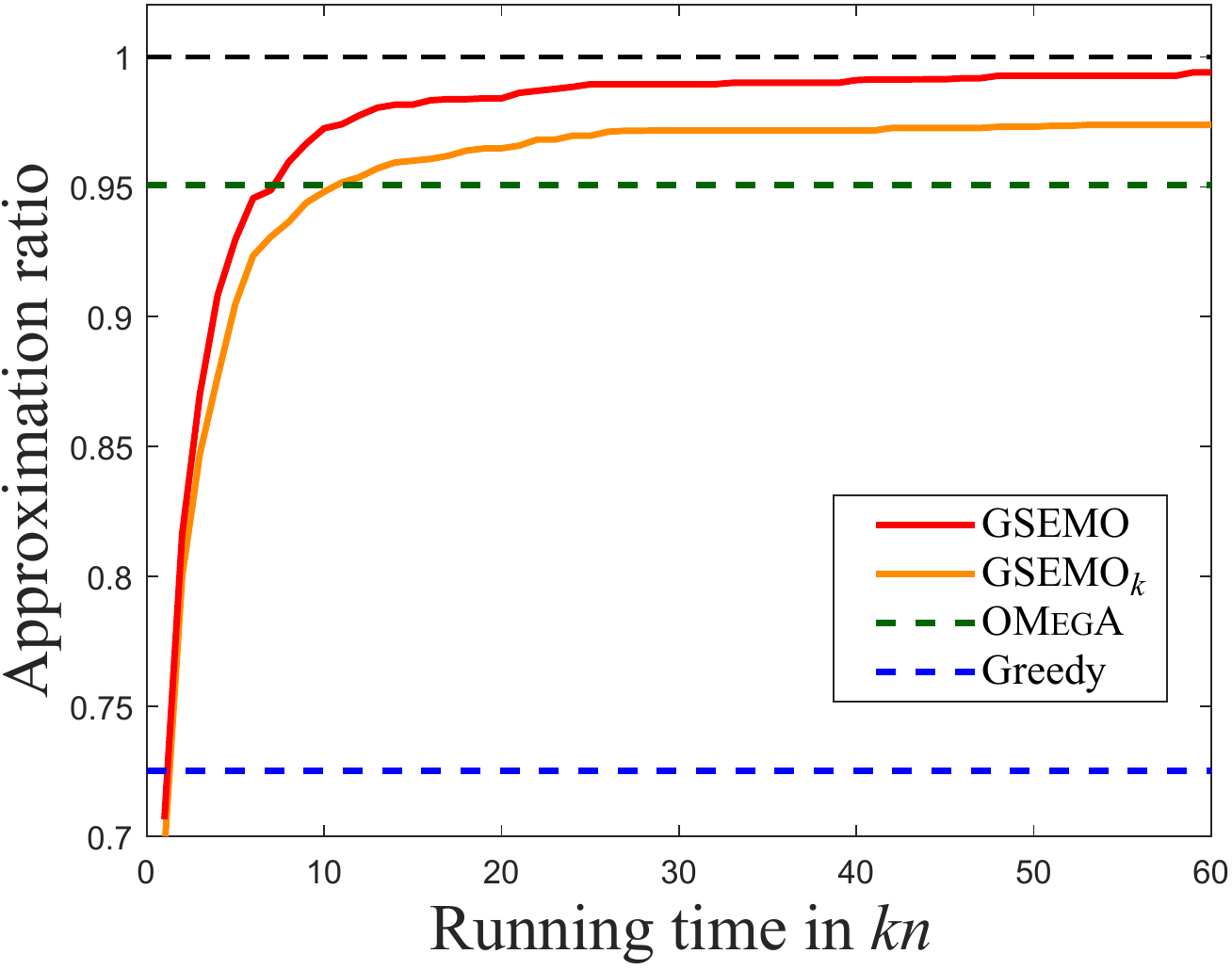}
\end{minipage}\hspace{2em}
\begin{minipage}[c]{0.42\linewidth}\centering
        \includegraphics[width=1\linewidth]{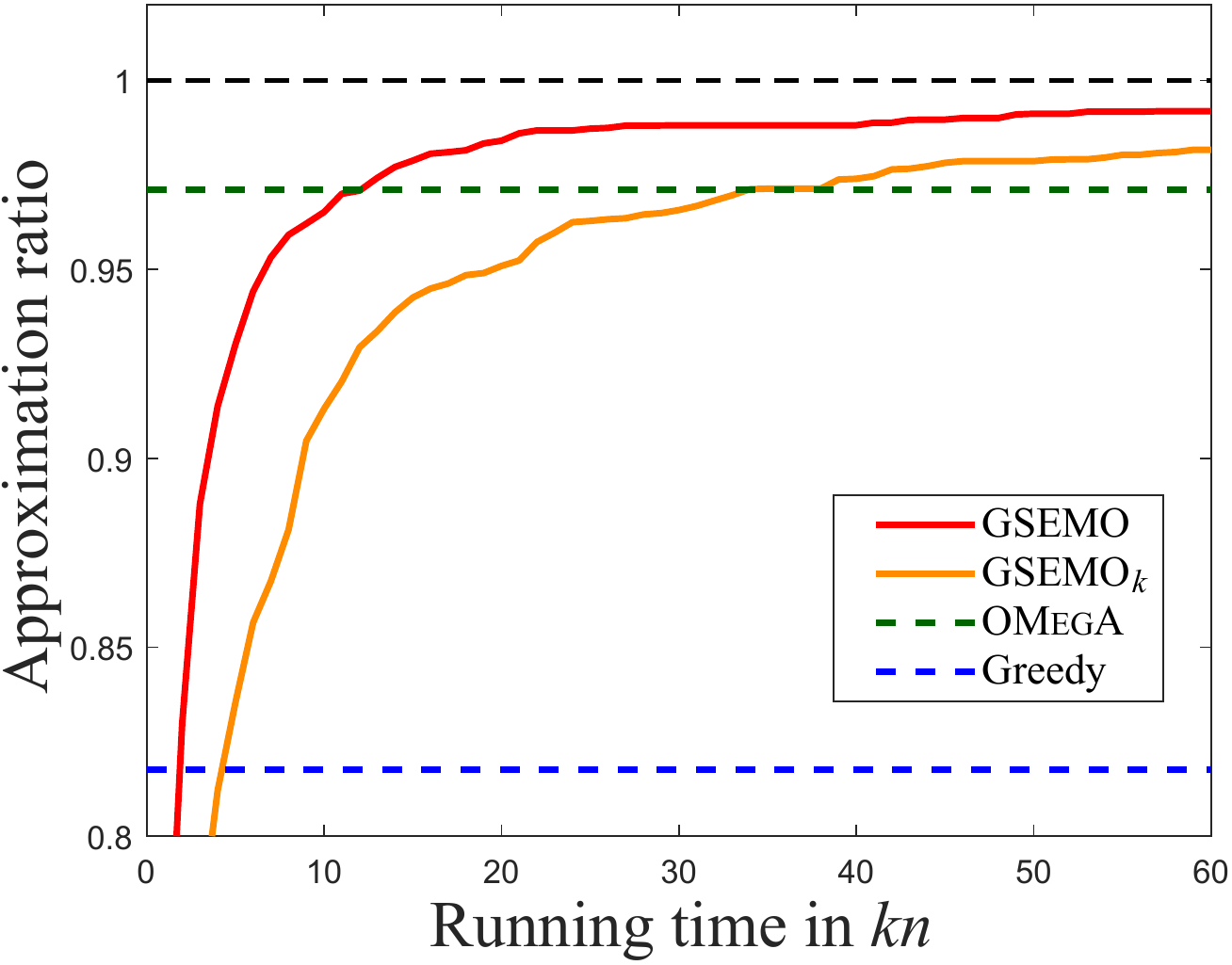}
\end{minipage}\\\vspace{0.3em}
\begin{minipage}[c]{0.42\linewidth}\centering
    \small(a) modular $h$, $d=5$
\end{minipage}\hspace{2em}
\begin{minipage}[c]{0.42\linewidth}\centering
    \small(b) submodular $h$, $d=5$
\end{minipage}
\caption{Approximation ratio vs. running time (i.e., number of objective evaluations) of the GSEMO and GSEMO$_k$ on maximizing DAG monotone submodular functions using the synthetic data set.}\label{fig-time}
\end{figure*}

We also examine the running time in CPU seconds. For the GSEMO, we record the running time until finding a solution at least as good as that obtained by the \textsc{OMegA}. The results are shown in Figure~\ref{fig-cputime}. As expected, the greedy algorithm is the fastest, the \textsc{OMegA} second, and the GSEMO costs the most time. This implies that the GSEMO can achieve better optimization performance by using more running time. As modern computer facilities have more powerful computing abilities, the GSEMO may have wide applicability.

\begin{figure*}[h!]\centering
\begin{minipage}[c]{0.42\linewidth}\centering
        \includegraphics[width=1\linewidth]{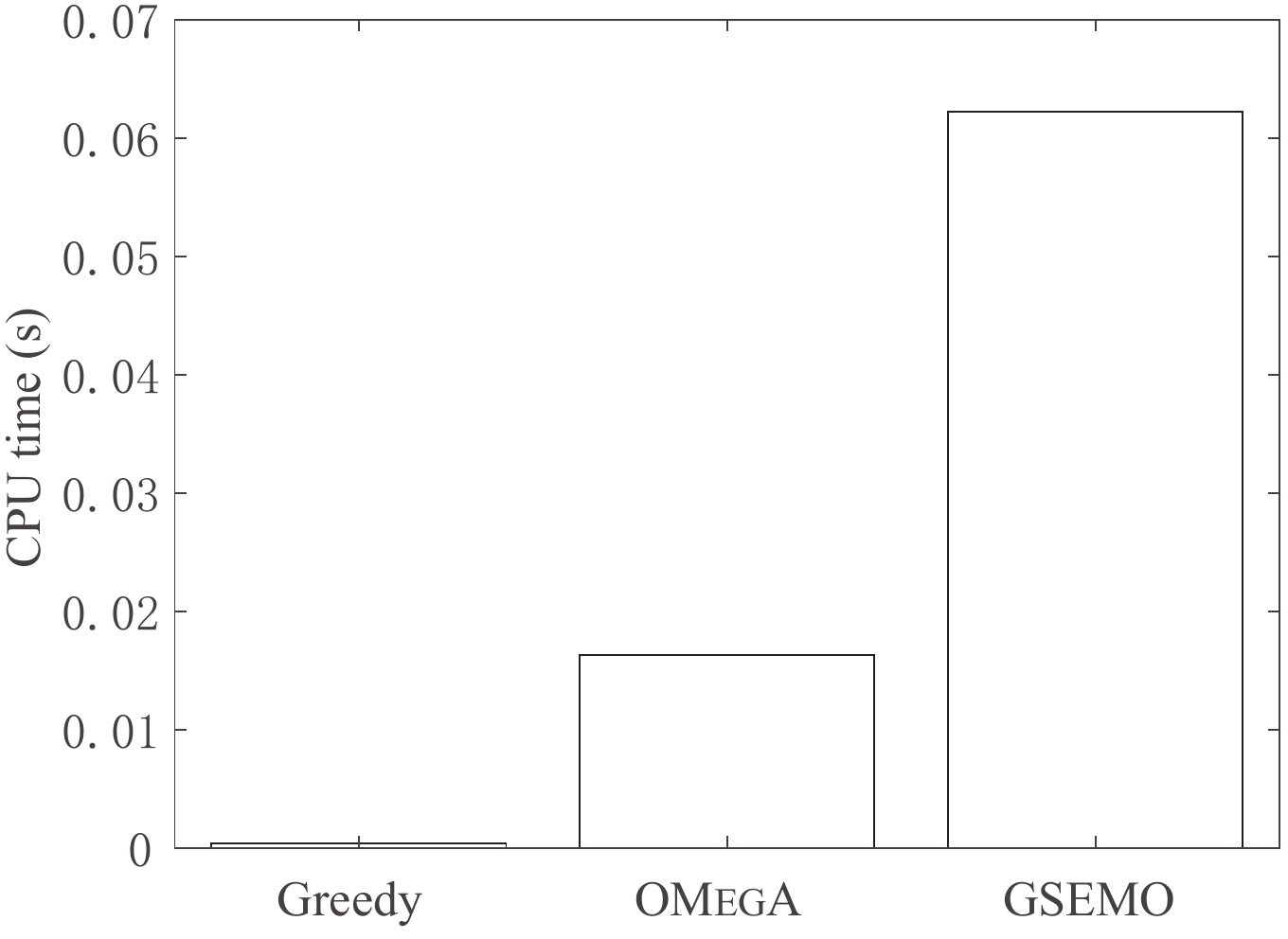}
\end{minipage}\hspace{2em}
\begin{minipage}[c]{0.42\linewidth}\centering
        \includegraphics[width=1\linewidth]{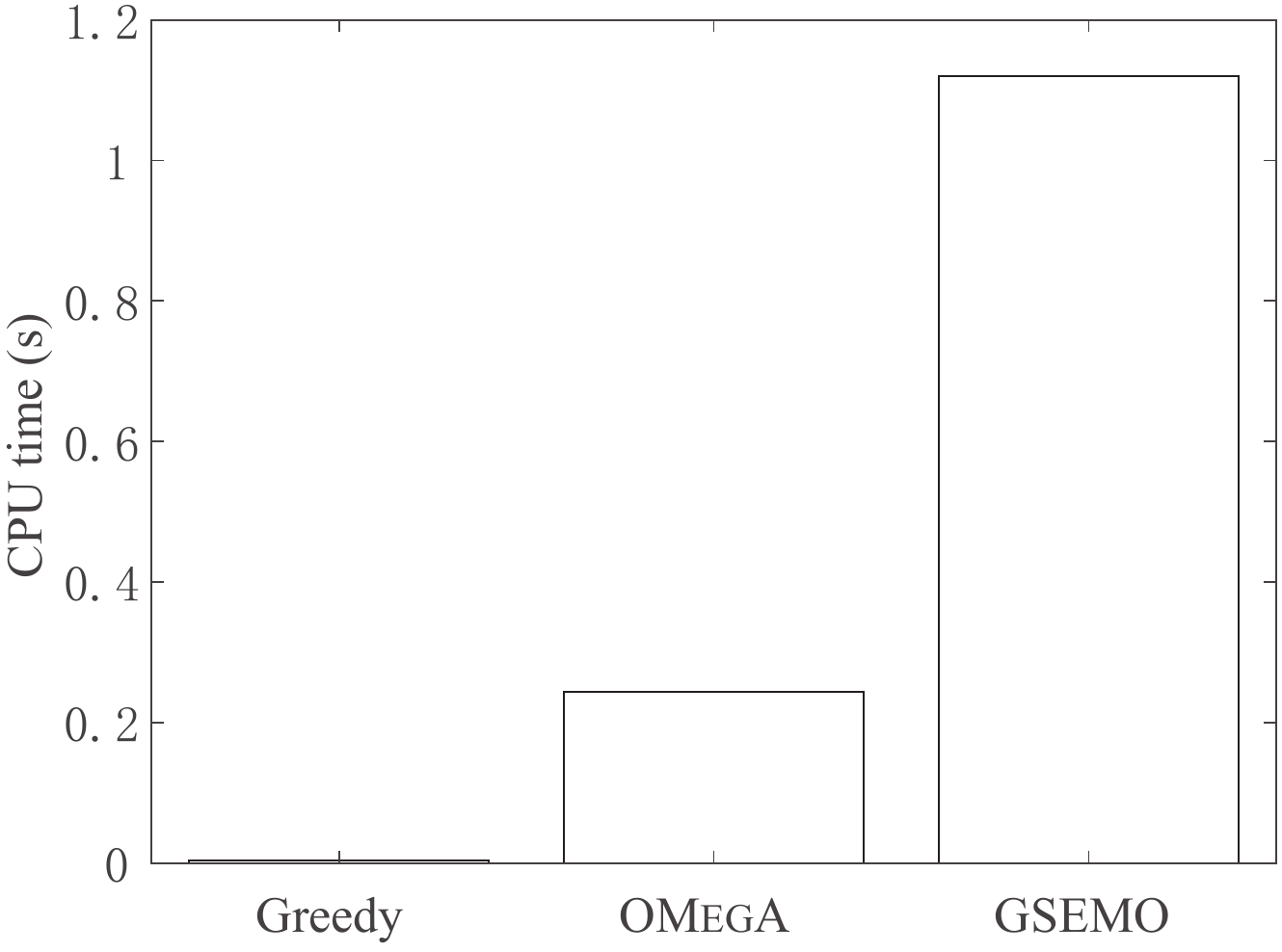}
\end{minipage}\\\vspace{0.3em}
\begin{minipage}[c]{0.42\linewidth}\centering
    \small(a) modular $h$, $d=5$
\end{minipage}\hspace{2em}
\begin{minipage}[c]{0.42\linewidth}\centering
    \small(b) submodular $h$, $d=5$
\end{minipage}
\caption{CPU running time of the greedy algorithm, the \textsc{OMegA} and the GSEMO on maximizing DAG monotone submodular functions using the synthetic data set.}\label{fig-cputime}
\end{figure*}

\section{Conclusion}\label{sec-conclusion}

This paper theoretically studies the approximation performance of EAs for solving the problem classes of maximizing monotone submodular functions over sequences, i.e., selecting a sequence with limited length that maximizes some given monotone submodular objective function. Different kinds of monotone submodular functions over sequences have been previously studied, including prefix monotone submodular functions, weakly monotone and strongly submodular functions, and DAG monotone submodular functions. For these cases, the greedy algorithm, the generalized greedy algorithm and the \textsc{OMegA} algorithm achieve the best-known polynomial-time approximation guarantee, respectively. We prove that within polynomial expected running time, a simple multi-objective EA called GSEMO can always reach or improve the best known approximation guarantee for these previously studied problem classes, providing a theoretical explanation for the generally good practical performance of EAs.

From the analysis, we can find the reason for the generally good approximation performance of the GSEMO. For maximizing prefix monotone submodular functions, the greedy algorithm needs to append an item to the end of the current sequence. For maximizing weakly monotone and strongly submodular functions, the generalized greedy algorithm needs to insert an item into a specific position of the current sequence. For maximizing DAG monotone submodular functions, the \textsc{OMegA} algorithm may need to insert two items into the current sequence. All of these behaviors can be accomplished by the common mutation operator in Definition~\ref{def:mutate}, and thus the GSEMO can be generally good. Note that our main goal is to show the existence of an EA (e.g., the GSEMO as we have shown in the paper) which can generally perform well for maximizing monotone submodular functions over sequences. There may be other EAs that also work well and even be better than the GSEMO, which can be studied in the future.

We also perform experiments on diverse applications of maximizing monotone submodular functions over sequences, and the results show the excellent performance of the GSEMO. But one can see that most of the experiments are performed on randomly generated problem instances. It would be interesting to test the algorithms on more real-world data sets. The superior empirical performance of the GSEMO over the greedy-style algorithms also implies that the GSEMO may be able to achieve a better approximation guarantee, and thus the tightness of the currently derived approximation guarantees is worth to be studied. The last problem class of maximizing DAG monotone submodular functions we studied requires a directed acyclic graph to capture the ordered preferences among items. It is also interesting to study whether EAs can still achieve good approximation guarantees when the directed acyclic graph is relaxed to a hypergraph~\citep{benouaret2019efficient,mitrovic2018submodularity}. Another interesting future work is to study the performance of EAs under the settings of adaptive sequence submodularity~\citep{mitrovic2019adaptive}.

\section{Acknowledgments}

The authors want to thank the associate editor and anonymous reviewers for their helpful comments and suggestions. This work was supported by the National Science Foundation of China (62022039, 62276124, 61921006), and the Shenzhen Peacock Plan (Grant No. KQTD2016112514355531).

\bibliography{tcs-seqselMOEA}
\bibliographystyle{abbrvnat}

\end{document}